\documentclass[letterpaper]{article} 
\usepackage{aaai2026}  

\usepackage{times}  
\usepackage{helvet}  
\usepackage{courier}  
\usepackage[hyphens]{url}  
\usepackage{graphicx} 
\usepackage{natbib}  
\usepackage{caption} 
\usepackage{algorithm}

\usepackage{newfloat}
\usepackage{listings}

\usepackage{amsmath}
\usepackage{microtype}

\usepackage{inconsolata}

\usepackage{graphicx}
\usepackage{todonotes}
\usepackage{subcaption}
\usepackage{bm}
\usepackage{amsmath}
\usepackage{tikz}
\usepackage{booktabs} 
\usepackage{times}
\usepackage{latexsym}
\usepackage{multirow}
\usepackage{bm}
\usepackage{enumitem}
\usepackage{microtype}
\usepackage{amsfonts}
\usepackage{graphicx}
\usepackage{float}
\usepackage{color}
\usepackage{colortbl}
\usepackage{xcolor}         
\usepackage{color}
\usepackage{colortbl}
\usepackage{nicefrac}
\definecolor{cblue}{rgb}{0.21,0.49,0.74}
\definecolor{pos}{rgb}{0.57,0.73,0.93}
\definecolor{neg}{rgb}{0.93,0.74,0.78}

\usepackage{color}
\usepackage{colortbl}
\usepackage{xcolor}         
\usepackage{color}
\usepackage{colortbl}
\usepackage{pifont}
\usepackage{makecell}
\usepackage{mathtools}
\usepackage{fontawesome}
\usepackage[T1]{fontenc}
\usepackage[capitalize]{cleveref} 

\usepackage{xspace}
\usepackage{amsmath}
\usepackage{amssymb}
\usepackage{mathtools}
\usepackage{amsthm}
\usepackage{multirow}
\usepackage{graphicx}
\usepackage{lscape}
\usepackage{algorithm}
\usepackage{algpseudocode}

\theoremstyle{plain}
\newtheorem{theorem}{Theorem}[section]

\theoremstyle{definition}

\theoremstyle{remark}

\DeclareCaptionStyle{ruled}{labelfont=normalfont,labelsep=colon,strut=off} 
\floatstyle{ruled}
\newfloat{listing}{tb}{lst}{}
\floatname{listing}{Listing}
\title{MHA2MLA-VLM: Enabling DeepSeek's Economical Multi-Head Latent Attention across Vision-Language Models}

\makeatletter
\let\AAAI@orig@copyright@text\copyright@text
\renewcommand{\copyright@text}{%
  \begingroup
    \footnotesize
    \def\@thefnmark{\faEnvelope}%
    \@makefntext{Corresponding authors are Tao Ji and Tao Gui.}%
    \par
  \endgroup
  \AAAI@orig@copyright@text
}
\makeatother

\author{
Xiaoran Fan\textsuperscript{1}\thanks{~Equal contribution.}, 
Zhichao Sun\textsuperscript{1}\footnotemark[1],
Tao Ji\textsuperscript{1}\footnotemark[1],
Lixing Shen\textsuperscript{2},
\bf Tao Gui\textsuperscript{1,3,4}
}

\affiliations{
    \textsuperscript{1}Fudan University 
    \\
\textsuperscript{2}Hikvision Inc 
    \\
    \textsuperscript{3}Shanghai Innovation Institute 
    \\
    \textsuperscript{4}Pengcheng Laboratory
    \\
{
{xrfan23@m.fudan.edu.cn; \{taoji,~tgui\}@fudan.edu.cn}

}
}

\usepackage{bibentry}

\begin{document}

\maketitle


\begin{abstract}


As vision-language models (VLMs) tackle increasingly complex and multimodal tasks, the rapid growth of Key-Value (KV) cache imposes significant memory and computational bottlenecks during inference. 
While Multi-Head Latent Attention (MLA) offers an effective means to compress the KV cache and accelerate inference, adapting existing VLMs to the MLA architecture without costly pretraining remains largely unexplored. 
In this work, we present \textbf{MHA2MLA-VLM}, a parameter-efficient and multimodal-aware framework for converting off-the-shelf VLMs to MLA. 
Our approach features two core techniques: (1) a modality-adaptive partial-RoPE strategy that supports both traditional and multimodal settings by selectively masking nonessential dimensions, and (2) a modality-decoupled low-rank approximation method that independently compresses the visual and textual KV spaces. 
Furthermore, we introduce parameter-efficient fine-tuning to minimize adaptation cost and demonstrate that minimizing output activation error, rather than parameter distance, substantially reduces performance loss. 
Extensive experiments on three representative VLMs show that MHA2MLA-VLM restores original model performance with minimal supervised data, significantly reduces KV cache footprint, and integrates seamlessly with KV quantization. 
\end{abstract}

\begin{links}
\link{Code \& Appendix}
{https://github.com/JT-Ushio/MHA2MLA-VLM}
\end{links}

\section{Introduction}
\label{sec:intro}

The Key-Value (KV) cache stores the complete contextual information required by large language models (LLMs), enabling efficient and accurate decoding of the current token. 
As the tasks handled by LLMs become increasingly complex (e.g. multimodal tasks~\cite{bordes2024introductionvisionlanguagemodeling} and deep thinking~\cite{pan2025surveyslowthinkingbasedreasoning}), the context length correspondingly increases. 
This results in a rapid expansion of the KV cache, which not only occupies large GPU memory but also leads to severe memory access bottlenecks due to the quadratic complexity of the standard attention mechanism~\cite{DBLP:conf/alt/KelesWH23}. 
Consequently, efficient inference in LLMs, especially in vision-language models (VLMs) with multimodal contexts, urgently requires cost-effective KV cache management and attention architectures.

A series of studies have identified redundancies in the KV cache~\cite{DBLP:journals/tmlr/0002LTTXCHD0025}. 
In terms of sequence length~\cite{DBLP:conf/nips/Zhang00CZC0TRBW23,DBLP:conf/emnlp/OrenHNA024}, KV cache pruning removes irrelevant tokens from the cache. 
Regarding representation precision~\cite{badri2023hqq}, KV cache quantization reduces the precision of vector representations. 
In the vector dimension, modifications such as Grouped/Multi-Query Attention (GQA and MQA) restructure the attention mechanism by enabling a single KV pair to be shared among a group of queries~\cite{emnlp/AinslieLJZLS23,shazeer2019fasttransformerdecodingwritehead}.

DeepSeek introduced Multi-Head Latent Attention (MLA), an advanced attention mechanism employing low-rank key-value joint compression~\cite{deepseekai2024deepseekv2strongeconomicalefficient}. 
Empirical results show that MLA outperforms standard Multi-Head Attention (MHA, \citeyear{c:22}) and its variants, while significantly reducing the KV cache size during inference, thereby enhancing inference efficiency.
\citeauthor{DBLP:conf/acl/JiGWGSCQZG25}~\shortcite{DBLP:conf/acl/JiGWGSCQZG25} proposed MHA2MLA, demonstrating that LLMs originally trained with MHA/GQA can be adapted to leverage MLA during inference. 
However, whether VLMs can undergo a similar transition to the MLA architecture remains an open question.

\begin{figure*}[t]
    \centering
    \includegraphics[width=1.0\linewidth]{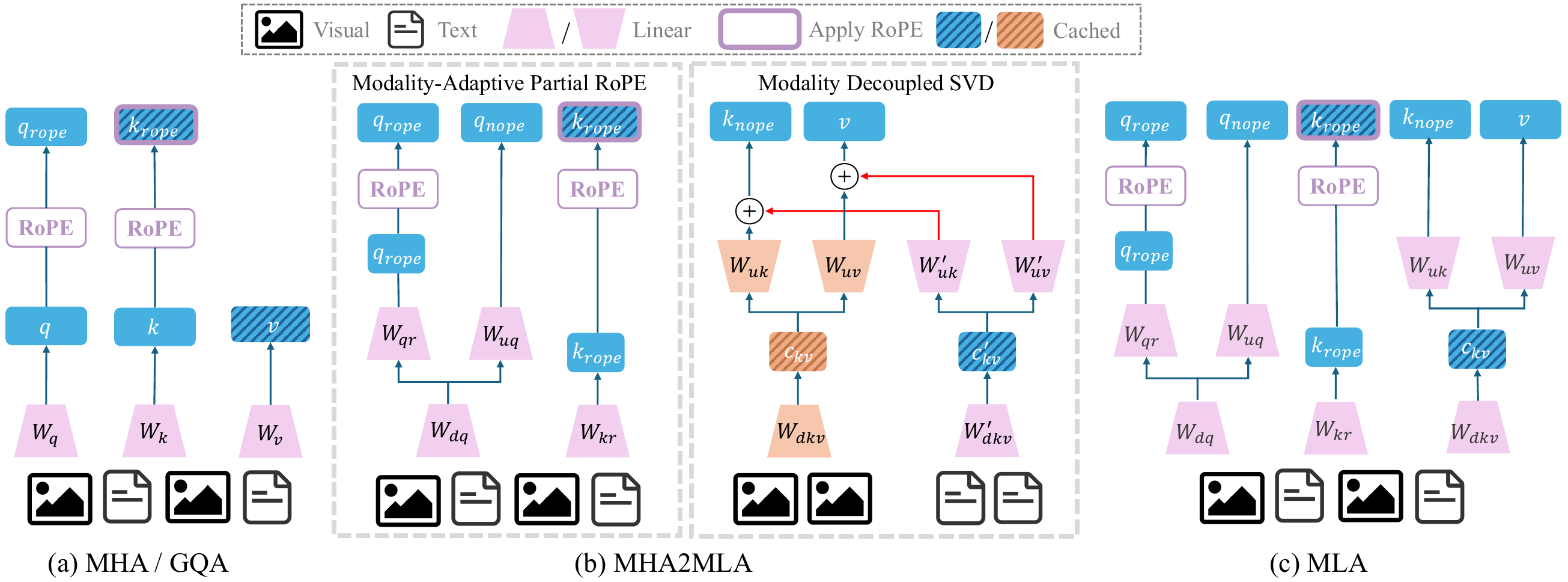}
    \caption{
The overview process of converting VLMs from MHA/GQA to MLA using MHA2MLA-VLM. Our method makes the attention inputs match MLA exactly, and low rank compression of the KV cache is consistent with MLA. The modality-decoupled design reduces truncation loss and maximizes the parameter reuse of pretrained weights.
}
    \label{fig:fig_main}
\end{figure*}

MHA2MLA involves two key steps: partial-rope conversion and KV joint low-rank approximation. 
For partial-rope, text-only LLMs have demonstrated that significantly reducing (e.g., -87.5\%) less important rotary frequencies requires only minimal fine-tuning to recover performance. 
In the case of VLMs, it is necessary to verify whether the retained rotary frequencies are equally effective for both image and text tokens. 
We address this question and further extend the method to multimodal rope (e.g., used in Qwen2-VL series). 

For KV joint low-rank approximation, inspired by SVDLLM V2~\cite{DBLP:conf/naacl/WangAWSZ25}, we improve the approximation from being applied to parameters ($\min||W - W^\prime||_F$, where $W^\prime$ is the low-rank approximation) to output activations ($\min||XW - XW^\prime||_F$). 
This enhancement significantly reduces performance degradation and the amount of fine-tuning data required. 
Moreover, we observe that the low-rank spaces of image and text tokens are orthogonal, necessitating separate low-rank approximations for each modality.

To reduce the cost of MHA2MLA-VLM adaptation, we introduce parameter-efficient fine-tuning (PEFT, \citeyear{xu2023parameterefficientfinetuningmethodspretrained}). 
During the partial-rope phase, only the two projection matrices for query and key are fine-tuned, while all other parameters are frozen. 
For the low-rank approximation phase, only the parameters within MLA are fine-tuned. 
It reduces the time required by 59\% (e.g., the MHA2MLA-VLM of Qwen2.5-VL is shortened from 22 hours to 9 hours). 
We validate the effectiveness of MHA2MLA-VLM on three representative models: LLaVA-1.5~\shortcite{DBLP:conf/cvpr/LiuLLL24}, LLaVA-NeXT~\shortcite{liu2024llavanext}, and Qwen2.5-VL~\shortcite{bai2025qwen25vltechnicalreport}. 
Furthermore, using LLaVA-NeXT, we demonstrate that MLA outperforms the KV cache pruning baseline and integrates seamlessly with KV quantization.

Our main contributions are:
\begin{itemize}[leftmargin=*,itemsep=0pt, topsep=0pt, parsep=0pt]
\item We successfully extend the MHA2MLA adaptation from text-only LLMs to VLMs, designing multimodal partial-rope and low-rank approximation algorithms.
\item By incorporating SVDLLM V2's minimization of output activation error and introducing PEFT, we significantly reduce the performance degradation and fine-tuning cost.
\item We demonstrate the effectiveness of MHA2MLA-VLM in three main VLMs with distinct architectures and demonstrate that it integrates seamlessly with KV quantization.
\end{itemize}
\section{MHA2MLA-VLM}
\label{sec:method}

Figure~\ref{fig:fig_main} provides an overview of MHA2MLA-VLM, we will describe the details of the two main components of MHA2MLA-VLM: (1) a modality-adaptive partial-RoPE strategy that supports both traditional and multimodal settings by selectively masking nonessential dimensions, and (2) a modality-decoupled low-rank approximation method that independently compresses the visual and textual KV spaces.

\subsection{Multimodal Partial-RoPE}
\label{ssec:mm_partial_rope}


To enable efficient migration of MHA/GQA-based VLMs to MLA, we introduce Multimodal Adaptive Partial-RoPE, which adaptively retains the most informative rotary dimensions according to the nature of the input from different modalities, achieving efficient architecture migration.

Current research on partial-RoPE has been limited to unimodal settings. For example, studies such as~\cite{gpt-neo,barbero2025roundroundgomakes} trained partial-RoPE models from scratch, achieving slightly better perplexity compared to full-RoPE. 
More recent work~\cite{DBLP:conf/acl/JiGWGSCQZG25} has explored adapting pre-trained full-RoPE models to partial-RoPE without costly retraining. However, their analyzes are strictly limited to LLM. 
In multimodal scenarios, for the input, visual and textual information is interleaved; For the VLMs' forward calculation, visual and text information are jointly entangled in the RoPE dimension. Simply applying the text-based partial RoPE strategy leads to suboptimal allocation of the retention frequency subspace, as visual and textual information exhibit distinct dimensional characteristics are ignored.
To overcome these limitations, retain the most informative rotation dimension based on modality-awareness, thereby enabling low-cost and efficient architectural transfer from MHA/GQA to MLA.

\paragraph{Full Vanilla RoPE} is a mechanism~\cite{DBLP:journals/ijon/SuALPBL24} for encoding positional information into queries and keys through frequency-specific rotations. Formally, given a query \(\bm{q}_i \in \mathbb{R}^{d_h}\) and a key \(\bm{k}_i \in \mathbb{R}^{d_h}\), we split them into 2D chunks:  
\[
\bm{q}_i,\bm{k}_i = \left[\bm{q}_i^{[2k, 2k+1]}\right]_{0 \leq k < \frac{d_h}{2}}, \left[\bm{k}_i^{[2k, 2k+1]}\right]_{0 \leq k < \frac{d_h}{2}}.
\]  
Formally, for each 2D chunk \( \bm{q}_i^{[2k, 2k+1]} \) and \( \bm{k}_i^{[2k, 2k+1]} \), the rotation matrix at position \( i \) is defined as: 
\[
\bm{R}_i^{[2k, 2k+1]}(\theta_k) = 
\begin{bmatrix} 
\cos(i\theta_k) & -\sin(i\theta_k) \\ 
\sin(i\theta_k) & \cos(i\theta_k) 
\end{bmatrix},
\] 
where $\theta_{k} = \beta^{-2k/{d_h}}$ is the frequency of rotation applied to a specific $k$-th pair of  $\mathcal{G}_d\in[0, \tfrac{d_h}{2})$, and $\beta$ is the frequency base wavelength.
The vanilla RoPE defines a matrix $\bm{A}_{t_i,t_j}$ that represents the relative positional encoding between two positions $t_i$ and $t_j$ in a 1D sequence:

\[
\bm{A}_{t_i,t_j}=\left(\bm{q}_{t_i}\bm{R}_{t_i}\right){\left(\bm{k}_{t_j}\bm{R}_{t_j}\right)}^\top
= \bm{q}_{t_i}\bm{R}_{\Delta t}\bm{k}_{t_j}^\top,
\]

where $\Delta t=t_i-t_j$, $\bm{q}_{t_i}$ and $\bm{k}_{t_j}$ are query and key vectors at positions $t_i$ and $t_j$, the relative rotation matrix is $\bm{R}_{\Delta t}$.

\paragraph{Full Multimodal RoPE}


There are two common approaches to extending RoPE to multimodal LLMs. One approach directly applies standard RoPE by flattening the visual tokens and treating both text and visual tokens as a single 1D sequence~\cite{DBLP:conf/cvpr/LiuLLL24,zhu2025internvl3exploringadvancedtraining,bai2023qwenvlversatilevisionlanguagemodel}. The other considers the unique characteristics of each modality and extends RoPE to multimodal scenarios, such as M-RoPE~\cite{wang2024qwen2vlenhancingvisionlanguagemodels,bai2025qwen25vltechnicalreport,wei2025videoropemakesgoodvideo}.

Unlike the vanilla 1D-RoPE in LLMs, which
is limited to encoding one-dimensional positional information, M-RoPE effectively models the positional information of multimodal inputs. 
This is achieved by deconstructing the original rotary embedding into
three orthogonal components: temporal (\( t \)), height (\( h \)), and width (\( w \)). 
For text, these components utilize identical position
IDs, making M-RoPE functionally equivalent to standard 1D-RoPE. 
For images, the temporal of each visual index is held constant, while height and width IDs are assigned distinctly based on the token’s position in the image. 
For videos, treated as frame sequences, the temporal ID increases with each frame, while the height and width IDs follow the same assignment pattern as for images.

Formally, given a query vector \( \bm{q}_i \in \mathbb{R}^{d_h} \) and key vector \( \bm{k}_i \in \mathbb{R}^{d_h} \), we partition them into  per-modality components:

\[
\bm{q}_i = \left[\bm{q}_i^{[t]};\, \bm{q}_i^{[h]} ;\, \bm{q}_i^{[w]} \right], \bm{k}_i = \left[\bm{k}_i^{[t]};\, \bm{k}_i^{[h]};\, \bm{k}_i^{[w]} \right]
\]

For each components, the embeddings are rotated separately by using its corresponding 2D rotary position encoding:
$\mathcal{G}_t\in[0, \tfrac{d_h}{8})$,
$\mathcal{G}_h\in[\tfrac{d_h}{8}, \tfrac{5d_h}{16})$,
$\mathcal{G}_w\in[\tfrac{5d_h}{16}, \tfrac{d_h}{2})$.
The query embeddings after applying M-RoPE are computed as follows:
\begin{align*}
\bm{q}_{i, rope}^{[t]} = \left[\bm{R}_{p_t}^{[2k, 2k+1]}(\theta_k)\bm{q}_i^{[2k, 2k+1]}\right]_{k\in\mathcal{K}_t},  \\
\bm{q}_{i, rope}^{[h]} = \left[\bm{R}_{p_h}^{[2k, 2k+1]}(\theta_k)\bm{q}_i^{[2k, 2k+1]}\right]_{k\in\mathcal{K}_h},  \\
\bm{q}_{i, rope}^{[w]} = \left[\bm{R}_{p_w}^{[2k, 2k+1]}(\theta_k)\bm{q}_i^{[2k, 2k+1]}\right]_{k\in\mathcal{K}_w}.  
\end{align*}
Thus, applying M-RoPE to both queries and keys becomes:
\begin{align*}
&\bm{q}_{i,\text{rope}} =
 \bigl[\bm{q}_{i,\text{rope}}^{[t]};\,
       \bm{q}_{i,\text{rope}}^{[h]};\,
       \bm{q}_{i,\text{rope}}^{[w]}\bigr]\in\mathbb{R}^{d_h},\\
&\bm{k}_{i,\text{rope}} =
 \bigl[\bm{k}_{i,\text{rope}}^{[t]};\,
       \bm{k}_{i,\text{rope}}^{[h]};\,
       \bm{k}_{i,\text{rope}}^{[w]}\bigr]\in\mathbb{R}^{d_h}.
\end{align*}
Let \( \bm{q}_i, \bm{k}_i \in \mathbb{R}^{d_h} \) be the query and key vectors at position \( i \). 
The corresponding relative matrix $\bm{A}$ is computed as:
\[
\bm{A}_{(t_i,h_i,w_i),(t_j,h_j,w_j)}=\bm{q}_{(t_i,h_i,w_i)}\bm{R}_{\Delta t,\Delta h,\Delta w}\bm{k}_{(t_j,h_j,w_j)}^\top,
\]
where $\Delta t=t_i-t_j$, $\Delta h=h_i-h_j$, and $\Delta w=w_i-w_j$.

\paragraph{Multimodal Adaptive Partial-RoPE Strategies}

Given $r$ retained rotational subspaces ($r=\frac{d_r}{2}\ll$ total subspaces $\frac{d_h}{2}$), the aim is to select which \( r \) subspaces preserve RoPE/M-RoPE encoding.

Recent research~\cite{DBLP:conf/acl/JiGWGSCQZG25} has systematically compared four heuristic partial RoPE methods: High-Frequency Preservation, Low-Frequency Preservation, Uniform Sampling, and Head-wise 2-norm contribution.
Experiments showed that the head‑wise 2‑norm contribution performs better.
Specifically, for each head \( h \), Head-wise 2-norm contribution computes the mean 2-norm score for each subspace in an LLM over long sequences. Then rank all subspaces by their 2-norm score and select the top-$r$:

\begin{equation}
    \mathcal{S}_{\text{2-norm}}\!=\!\underset{0\le k<\frac{d_h}{2}}{\text{top-}r} \left( \left\|\mathbf{q}_*^{[2k,2k+1]}\right\|\left\|\mathbf{k}_*^{[2k,2k+1]}\right\| \right).
\end{equation}






The above method uses heuristic calculations to remove unimportant subspaces and retain r important ones. It cannot effectively reflect the impact of removing specific dimensions on the original whole.
Moreover, to enable the migration of VLMs from MHA/GQA to MLA, 
We propose Contribution-Aware Multimodal Partial-RoPE, based on KL-divergence (MKL). A data-driven and training-free strategy that extends frequency-subspace selection to multimodal inputs.

For each layer $l$ and each attention head $h$ we compute the \emph{frequency-wise KL sensitivity}:
\begin{equation}
\mathcal{I}^{l,h,k}= \mathbb{E}_{\mathcal{D}}\Bigl[\,\mathrm{KL}\bigl(\mathbf{P}_{l,h}^{\mathrm{full}}\parallel\mathbf{P}_{l,h,k}^{\mathrm{masked}}\bigr)\Bigr],
\end{equation}
where $\mathbf{P}_{\mathrm{full}}^{l,h}$ denotes the attention distribution produced by the original full RoPE/M-RoPE model and $\mathbf{P}_{l,h,k}^{\mathrm{masked}}$ is obtained after zero-ablating the $k$-th subspace in the query and key projections of head $h$.  
A large $\mathcal{I}^{l,h,k}$ indicates that subspace $k$ is critical for positional understanding under the current multimodal inputs.
The ${d_h}/{2}$ subspaces are then ranked in descending order and the top-$r$ indices are retained:
\begin{equation}
\mathcal{S}_{\mathrm{MKL}}^{l,h}=\underset{0\le k<\frac{d_h}{2}}{\mathrm{top}\text{-}r}\ \mathcal{I}^{l,h,k}.
\end{equation}


Thus, our modality-adaptive strategy preserves rotation-critical subspaces. ~ \Cref{tab:compare_2norm} shows that our Multimodal Adaptive Partial-RoPE strategy (MKL) outperforms the strongest baseline ($\mathcal{S}_{\text{2-norm}}$).
We will analyze the effectiveness of the strategy in \Cref{sec:2norm_analysis}.
Based on our proposed method, non-selected subspaces (\( k \notin \mathcal{S} \)) become NoPE dimensions, enabling seamless integration with MLA's latent compression.

\begin{algorithm}[t]
\captionsetup{font=small}
\caption{Pseudocode of Modality-Decoupled SVD}
\small
\begin{algorithmic}[1]
\Statex \textbf{Input:} $W$ — Joint KV weight matrix 
\Statex \hspace{3em}$\mathbf X_{{visual}},\;\mathbf X_{{text}}$ — Visual / Text activations
\Statex \hspace{3em}$r_{{visual}},\;r_{{text}}$ — Target ranks
\Statex \textbf{Output:} $\bigl\{\mathbf W^{up}_m,\mathbf W^{down}_m\bigr\}_{m\in\{{visual},{text}\}}$
\Procedure{MD\_SVD}{$W,\mathbf X_{{visual}},\mathbf X_{{text}},r_{{visual}},r_{{text}}$}
    \For{$m\in\{{Visual},{Text}\}$}                            \Comment{process each modality}
        \State $\mathbf S_m \gets \mathbf X_m\,\mathbf X_m^{\top}$                 \label{ln:cov}
        \State $[\mathbf U_s,\boldsymbol\Sigma_s,\_] \gets \operatorname{SVD}(\mathbf S_m)$
        \State $\mathbf D \gets W\,\mathbf U_s\,\boldsymbol\Sigma_s^{1/2}$         \label{ln:mid}
        \State $[\mathbf U_d,\boldsymbol\Sigma_d,\mathbf V_d] \gets \operatorname{SVD}(\mathbf D)$
        \State Keep first $r_m$ components of $(\mathbf U_d,\boldsymbol\Sigma_d,\mathbf V_d)$
        \State $\mathbf W^{up}_m \gets \mathbf U_d\,\boldsymbol\Sigma_d^{1/2}$ \label{ln:up}
        \State $\mathbf W^{down}_m \gets
               \boldsymbol\Sigma_d^{1/2}\mathbf V_d\,
               \boldsymbol\Sigma_s^{-1/2}\,\mathbf U_s^{-1}$           \label{ln:down}
    \EndFor
    \State \Return $\bigl\{\mathbf W^{up}_m,\mathbf W^{down}_m\bigr\}_{m}$ \label{ln:return}
\EndProcedure
\end{algorithmic}
\label{algo:md-svd}
\end{algorithm}

\subsection{Modality-Decoupled SVD (MD-SVD)}
\label{sec:md-svd}

After transforming the VLMs' Vanilla RoPE/M-RoPE to modality-adaptive partial RoPE, we get the first component $\bm{k}_{rope}$. The next step is to construct $\bm{c}^{(m)}_{kv}\;\in\;\mathbb{R}^{d_{kv}}$, a low-rank joint embedding of the modality-specific \(\bm{k}^{(m)}_{{nope}}\) and \(\bm{v}^{(m)}\).







\paragraph{Unimodal SVD Baselines}
Studies on SVD-driven LLM compression can be grouped into two paradigms.
The first paradigm operates directly on the models' weights~\cite{DBLP:conf/iclr/HsuHCLSJ22} reduces truncation loss by estimating weight importance and preserving more important weights. ~\cite{DBLP:conf/acl/JiGWGSCQZG25} jointly optimizes the latent space for both keys and values, which validates that joint factorization better preserves pre-trained knowledge.
The second paradigm~\citep{DBLP:conf/iclr/01200W025,DBLP:conf/naacl/WangAWSZ25} augments SVD with activation-aware transformations to lower the truncation loss $L$ in the form of Frobenius norm as follows during LLM compression:

\begin{equation}
    \mathcal{L}^2 = ||\bm W \bm X-\bm W' \bm X||_F^2
\end{equation}

Given a single-modality activation matrix $\mathbf X$, we first construct the covariance
$\bm S=\bm X\bm X^{\top}$, second, we performs a second 
SVD on $\bm D=
\bm W\bm U_s\bm\Sigma_s^{1/2}$.  
~\cite{DBLP:conf/naacl/WangAWSZ25} proves the compressed weight
\(
  \bm W'=\bm U_d\,
             \operatorname{Trunc.}(\bm\Sigma_d)\,
             \bm V_d\,
             \bm\Sigma_s^{-1/2}\bm U_s^{-1}
\)
achieves the theoretical minimum truncation loss.


\paragraph{Modality-Decoupled SVD}
\paragraph{Motivation}
Visual and text activations exhibit distinct scales,
estimating a single covariance on
$\bm{S}_{visual}$ or $\bm{S}_{text}$ causes the dominant modality to distort the singular value distribution, reducing the quality of other modality after SVD.
Our aim is to keep the \emph{shared} KV weight $W$ intact while
deriving two modality-decoupled low-rank projections that jointly
minimise truncation loss. The pseudocode of our proposed MD-SVD is provided in~\Cref{algo:md-svd}.

\begin{theorem} \label{throrem:md_svd}
For a multimodal token sequence, we denote the joint activation matrix as $\bm{X}_{joint}\in\mathbb{R}^{d\times (n_v+n_t)}$ which can be partitioned into two modality-specific components:
$$\bm{X}_{joint}=[\bm{X}_{visual};\bm{X}_{text}]$$
where $\bm{X}_{visual}\in\mathbb{R}^{d\times n_v}$ and $\bm{X}_{text}\in\mathbb{R}^{d\times n_t}$ denote the activations corresponding to visual and text tokens, respectively.

Then the minimum loss of joint-modals is larger than or equal to the sum of the minimum losses of split-modals:
\begin{equation}
    \min\mathcal{L}^{2}_{joint}\ge \min\mathcal{L}^2_{visual} + \min\mathcal{L}^2_{text}
\end{equation}

\end{theorem}

\begin{proof}
First, Given $\bm X_{joint}$ and weight matrix $\bm W$, inspired by ~\cite{DBLP:conf/naacl/WangAWSZ25}, we can obtain an optimal matrix $\bm W'$ that minimizes the joint loss $\mathcal{L}_{joint}^2:$

\begin{equation}
    \min\mathcal{L}^{2}_{joint}=|| \bm W \bm X_{joint}-\bm W' \bm X_{joint}||_F^2
\end{equation}

Second, by decomposing $X_{joint}$ into its constituent modalities, we obtain the following:

\begin{equation}
    \begin{aligned}
        \min\mathcal{L}^{2}_{joint}
        &=
        \left\lVert
        \bm W[\bm X_{visual};\bm X_{text}]
        \right. \\
        &\quad
        \left.
        -\,\bm W^{\prime}[\bm X_{visual};\bm X_{text}]
        \right\rVert_F^2 \\
        &=
        \left\lVert\bm W\bm X_{visual}-\bm W^{\prime}\bm X_{visual}\right\rVert_F^2 \\
        &\quad+
        \left\lVert\bm W\bm X_{text}-\bm W^{\prime}\bm X_{text}\right\rVert_F^2 \\
        &\ge \min\mathcal{L}^2_{visual} + \min\mathcal{L}^2_{text}.
    \end{aligned}
\end{equation}

This inequality holds because joint optimization imposes shared weights between different modalities, which limits the upper bound of model optimization. In contrast, our proposed Modality-Decoupled SVD (MD-SVD) optimizations allow for separate weights, enabling each branch to independently minimize its truncation loss. We provide a quantitative analysis in~\Cref{sec:svd_proof} to further validate this theoretical insight.

\end{proof}

\begin{table*}[th]
    \centering
    \small
    \setlength{\tabcolsep}{1.8mm} 
    \begin{tabular}{l@{}lccr@{\hspace{2pt}}lcccccccc}
      \toprule
      \multicolumn{2}{l}{\textbf{Model}} & \textbf{Tokens} & \textbf{KV Mem.} & \multicolumn{2}{c}{\textbf{Avg.}} & \textbf{AI2D} & \textbf{GQA} & \textbf{POPE} & \textbf{SEED$^{\text{I}}$} & \textbf{RWQ} & \textbf{MMB$^{\text{EN}}$} & \textbf{Chart} & \textbf{Doc} \\
      \midrule
      \rowcolor{gray!10} \multicolumn{2}{l}{${\text{LLaVA-1.5}_{7B}}$} & {} &   
      & \multicolumn{2}{l}{63.89} & 53.63 & 61.27 & 84.03 & 65.28 & 56.21 & 62.89 & - & - \\
      \textit{- MHA}& ~$d_{kv}\!=\!256$ &  &   
      & 64.13 & & 54.92 & 61.97 & 83.96 & 64.80 & 56.34 & 62.80 & - & - \\
      \multirow{3}{*}{\textit{- MHA2MLA}}   & ~$d_{kv}\!=\!64$ & {0.5B} & -62.50\%  
      & 63.84 & \textsubscript{-0.29} & 54.05 & 62.32 & 82.74 & 65.17 & 54.51 & 64.26 & - & - \\
      & ~$d_{kv}\!=\!32$ & {(0.025\%)} & -75.00\%  
      & 63.58 & \textsubscript{-0.55} & 52.49 & 62.06 & 83.18 & 65.04 & 54.90 & 63.83 & - & - \\
      & ~$d_{kv}\!=\!16$ &  & -81.30\% 
      & 62.27 & \textsubscript{-1.86} & 48.41 & 62.08 & 83.61 & 65.85 & 51.37 & 62.29 & - & - \\
      
    \midrule
    
    \rowcolor{gray!10} \multicolumn{2}{l}{${\text{LLaVA-NeXT}_{8B}}$} & {} &   
      & \multicolumn{2}{l}{70.72} & 70.56 & 65.13 & 87.18 & 72.46 & 58.69 & 71.82 & 68.44 & 71.47 \\
      \textit{- GQA}& ~$d_{kv}\!=\!256$ &  &   
      & 70.78 & & 70.76 & 65.30 & 87.22 & 71.95 & 59.22 & 71.74 & 69.04 & 70.99 \\
      \multirow{3}{*}{\textit{- GQA2MLA}}   & ~$d_{kv}\!=\!128$ & {1.8B} & -84.38\%  
      & 70.23 & \textsubscript{-0.55} & 69.53 & 65.25 & 85.96 & 71.90 & 59.08 & 72.08 & 68.68 & 69.34 \\
      & ~$d_{kv}\!=\!64$ & {(0.012\%)} & -90.63\% 
      & 68.75 & \textsubscript{-2.03} & 68.20 & 64.45 & 86.32 & 71.57 & 58.56 & 69.50 & 65.56 & 65.83 \\
      & ~$d_{kv}\!=\!32$ & {} & -93.75\%  
      & 66.72 & \textsubscript{-4.06} & 65.90 & 63.98 & 86.56 & 71.15 & 56.86 & 63.32 & 63.60 & 62.41 \\
      
    \midrule
    
    \rowcolor{gray!10} \multicolumn{2}{l}{${\text{Qwen2.5-VL}_{7B}}$} & {} &   
      & \multicolumn{2}{l}{79.47} & 82.58 & 60.42 & 86.22 & 77.34 & 68.37 & 82.82 & 83.24 & 94.78 \\
      \textit{- GQA}& ~$d_{kv}\!=\!256$ &  &   
      & 80.75 & & 83.35 & 63.26 & 87.50 & 76.83 & 69.67 & 84.28 & 86.88 & 94.29 \\
      \multirow{3}{*}{\textit{- GQA2MLA}}   & ~$d_{kv}\!=\!128$ & {0.5B} & -91.07\%  
      & 80.63 & \textsubscript{-0.12} & 82.71 & 63.26 & 87.54 & 76.83 & 69.41 & 84.02 & 86.80 & 94.48 \\
      & ~$d_{kv}\!=\!64$ & {(0.002\%)} & -94.64\%  
      & 79.47 & \textsubscript{-1.28} & 80.63 & 63.08 & 87.85 & 76.51 & 68.76 & 81.35 & 84.76 & 92.81 \\
      & ~$d_{kv}\!=\!32$ & {} & -96.43\%  
      & 77.02 & \textsubscript{-3.73} & 78.11 & 62.87 & 87.42 & 74.75 & 66.54 & 77.15 & 82.16 & 87.17 \\

      \bottomrule
    \end{tabular}
    \caption{Performance of three VLMs with different architectures (e.g., MHA2MLA, GQA2MLA) and $d_{kv}$. The eight benchmarks include AI2D (\citeyear{DBLP:conf/eccv/KembhaviSKSHF16}), GQA (\citeyear{DBLP:conf/cvpr/HudsonM19}), POPE (\citeyear{DBLP:conf/emnlp/LiDZWZW23}), SEED-Bench (SEED,~\citeyear{li2023seedbenchbenchmarkingmultimodalllms}), RealWorldQA (RWQ,~\citeyear{DBLP:conf/iclr/0004ZTFZWLWW00025}), MMBench (MMB,~\citeyear{DBLP:conf/eccv/LiuDZLZZYWHLCL24}), ChartQA (\citeyear{DBLP:conf/acl/MasryLTJH22}), DocVQA (\citeyear{DBLP:conf/wacv/MathewKJ21}).} 
\label{tab:main_result}
\end{table*}

\section{Experiment}
\label{sec:exper}

\paragraph{Setups} 

Details of the models, datasets, parameter-efficient and data-efficient strategies, evaluation, and hyperparameter setups are placed in Appendix A.

Based on the above settings, 
We conduct systematic experiments of MHA2MLA and GQA2MLA under various KV dimensions, our experiments address three critical questions:
\begin{enumerate}[leftmargin=*,itemsep=0pt, topsep=0pt, parsep=0pt]
    \item Can MHA2MLA-VLM maintain multimodal accuracy when training and inference are limited to a tight computation or data budget?
    \item What are the characteristics of SVD in multimodal scenarios using MHA2MLA-VLM?
    \item Comparison between MHA2MLA-VLM and cache pruning, and can it be combined with cache compression?
\end{enumerate}

\subsection{Main Results}
As shown in~\Cref{tab:main_result}, compared to the original VLMs after fine-tuning ($d_{kv}=256$), as $d_{kv}$ decreases, our method significantly reduces the KV cache memory with only slight performance drop. 
For example, for Qwen2.5-VL, our GQA2MLA method still achieves an overall performance of 79.47 even after reducing the KV memory by 94.64\%, which is comparable to the original model or the fine-tuned baseline model.

Second, our architecture migration method demonstrates both parameter and data efficiency. 
Compared to existing models training from scratch that rely on trillions of training tokens~\cite{bai2025qwen25vltechnicalreport}, our approach enables the migration from MHA to MLA architectures using only within 1.8B tokens. 
Furthermore, within our two-stage training framework, we fine-tune only \textbf{\(\sim\!\bm {10}\ \bm {\%} \)} of the VLM's parameters while still achieving competitive performance.
These results highlight the effectiveness of our method in enabling architectural adaptation with limited data or GPU resources, making it highly applicable to resource-constrained scenarios.

In addition, Figure~\ref{fig:dkv_loss} shows the parameter and data efficient fine-tuning losses of MHA2MLA and GQA2MLA in different compression ratios. 
Even with small data (0.002\%) and only $\sim 10\%$ parameter updates, training converges quickly. 
Greater compression widens the loss gap compared to the uncompressed baseline. When $d_{kv}$ is large (64 or 128), with KV cache compression of 84.38\% (LLaVA-NeXT) and 62.50\% (LLaVA-1.5), it achieves loss comparable to fine-tuning the original model.
Note that the changing trends of these loss curves are the same, indicating that our architecture transfer preserves the VLMs' internal knowledge to a large extent.

Overall, our MHA2MLA-VLM method generalizes across both MHA and GQA architectures, enabling parameter-efficient and data-efficient training for diverse VLMs.

\begin{figure}[t]
    \centering
\includegraphics[width=1.0\linewidth]{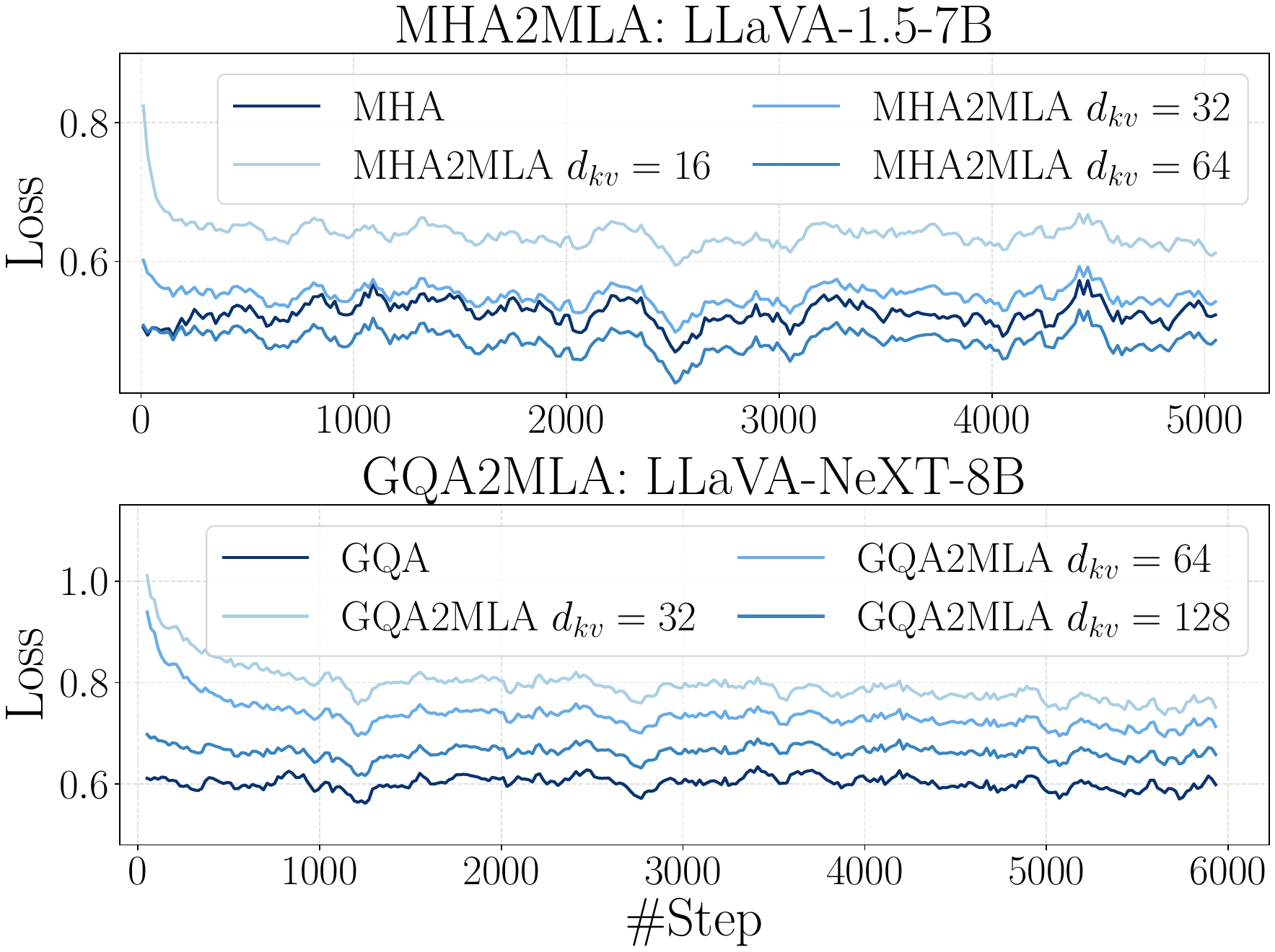}
    \caption{Training loss curves of \textbf{MHA2MLA} (LLaVA-1.5) and \textbf{GQA2MLA} (LLaVA-NeXT) with different $d_{kv}$ settings.}
    \label{fig:dkv_loss}
\end{figure}

\begin{table}[t]
\centering
\begin{tabular}{c@{\hskip 5pt}@{\hskip 5pt}p{1.5cm}@{\hskip 5pt}c@{\hskip 5pt}c}
  \toprule
  {\textbf{Model}} & \textbf{Type} & \textbf{KV Mem.} & \textbf{Avg.} \\
  \midrule
 \rowcolor{gray!10} \multicolumn{4}{c}{\textit{\textbf{Origin}}} \\

  $\text{LLaVA-NeXT}$ &
  BF16 & 100.0\% & 70.78 \\

 \rowcolor{gray!10} \multicolumn{4}{c}{\textit{\textbf{Cache Pruning}}} \\
  \arrayrulecolor{gray!20}
  \hline
  \multirow{6}{*}{$\text{LLaVA-NeXT}$}
  & H\textsubscript{2}O & \multirow{2}{*}{-81.25\%} & 61.43 \\
  & TOVA & & 54.53 \\
  & H\textsubscript{2}O & \multirow{2}{*}{-75.00\%} & 62.34 \\
  & TOVA & & 57.24 \\
  & H\textsubscript{2}O & \multirow{2}{*}{-62.50\%} & 63.38 \\
  & TOVA & & 60.48 \\
  
\rowcolor{gray!10} \multicolumn{4}{c}{\textit{\textbf{Cache Quantization}}} \\
\multirow{4}{*}{$\text{LLaVA-NeXT}$}
  & Int4$_{\text{Quanto}}$ & \multirow{2}{*}{-75.00\%} 
  & 70.53 \\
  & Int4$_{\text{HQQ}}$ & & 70.62 \\
  & Int2$_{\text{Quanto}}$ & \multirow{2}{*}{-87.50\%} 
  & 67.21 \\
  & Int2$_{\text{HQQ}}$ & & 60.71 \\

  \arrayrulecolor{black}
  \midrule
  \multirow{3}{*}{~$d_{kv}\!=\!128$} &
  BF16 & -37.50\% & 70.23 \\
  \arrayrulecolor{gray!20}
  \hline
  & Int4$_{\text{Quanto}}$ & \multirow{2}{*}{-84.38\%}  & \bf 70.22 \\
  & Int4$_{\text{HQQ}}$ & & \bf 70.21 \\

  \arrayrulecolor{black}
  \midrule
  \multirow{3}{*}{~$d_{kv}\!=\!64$} &
  BF16 & -62.50\% & \bf 68.75 \\
  \arrayrulecolor{gray!20}
  \hline
  & Int4$_{\text{Quanto}}$ & \multirow{2}{*}{-90.63\%}  & \bf 68.66 \\
  & Int4$_{\text{HQQ}}$ & & \bf 68.64 \\
  
  \arrayrulecolor{gray!20}
  \arrayrulecolor{black}
  \midrule
  \multirow{3}{*}{~$d_{kv}\!=\!32$} &  
  BF16 & -75.00\% & \bf 66.72 \\
  \arrayrulecolor{gray!20}
  \hline
  & Int4$_{\text{Quanto}}$ & \multirow{2}{*}{-93.75\%}  & \bf 66.71 \\
  & Int4$_{\text{HQQ}}$ &  & \bf 66.72 \\
  
  \arrayrulecolor{black}
    \bottomrule
\end{tabular}
\caption{Comparison of GQA2MLA with other cache compression strategies, including Cache Pruning and Cache Quantization for LLaVA-NeXT. Bolded scores indicate that GQA2MLA outperforms Cache Pruning or Int2 quantization methods at the same or higher compression levels.}
\label{tab:compare_cache}
\end{table}

\begin{table}[htb]
\centering
\small
\begin{tabular}{lccr@{\hspace{2pt}}c}
  \toprule
  \textbf{Model} & \textbf{$d_{kv}$} & \textbf{KV Mem.} & \multicolumn{2}{l}{\textbf{Avg}} \\
  \midrule
  \multirow{2}{*}{$\text{LLaVA-NeXT}_\text{GQA2MLA}$} & 32 & 93.75\% & 67.56 & \\
  & 64 & 90.63\% & 69.70 & \\
\midrule
\multirow{2}{*}{\textit{w/o Modality Decoupled}} & 32 & 93.75\% & 67.14 & \\
& 64 & 90.63\% & 69.25 & \\
\midrule
\multirow{2}{*}{\textit{w/o MD-SVD Init}}& 32 & 93.75\% & 68.21 & \\
& 64 & 90.63\% & 68.47 & \\
\midrule
\multirow{2}{*}{\textit{w/o Two Stage}} & 32 & 93.75\% & 67.03 & \\
& 64 & 90.63\% & 68.82 & \\
  \bottomrule
\end{tabular}
\caption{Ablation study on the core designs of GQA2MLA on LLaVA-NeXT, including Two Stage training, Modality Decoupled, and MD-SVD Initialization.}
\label{tab:ablation_study}
\end{table}

\subsection{MHA2MLA, Cache Pruning and Compression}
To demonstrate the efficiency of MHA2MLA-VLM in compressing KV cache, we compare it with representative pruning and quantization methods, as summarized in Table \ref{tab:compare_cache}.

First, at comparable compression ratios, MHA2MLA achieves significantly better performance than the cache pruning methods. For instance, at a 62.50\% reduction in KV cache, the average score reaches 68.75, while cache pruning methods such as H$_2$O and TOVA yield notably lower scores (e.g., 63.38 and 60.48, respectively). Similar trends can be observed under 75.00\% and 81.25\% compression ratios. 

Second, MHA2MLA can be effectively combined with cache quantization methods to achieve further compression without performance degradation. 
For all levels of $d_{kv}$ compression in GQA2MLA, applying Int4$_{\text{Quanto}}$ and Int4$_{\text{HQQ}}$, even at compression levels of 84.38\% ($d_{kv}$=128 with 70.22) and 90.63\% ($d_{kv}$=64 with 68.66), the performance is still significantly better than the Int2$_{\text{HQQ}}$ of 67.21 based on cache quantization at 87.5\%.
These results demonstrate that MHA2MLA and quantization techniques are highly compatible and can be effectively combined to achieve compounded compression benefits with minimal loss in performance.

These results validate that MHA2MLA achieves better performance compared to existing cache compression strategies, and can integrate seamlessly with cache quantization methods to achieve higher cache reduction.

\subsection{Ablation Study}

\begin{figure}[t]
    \centering
    \includegraphics[width=1.0\linewidth]{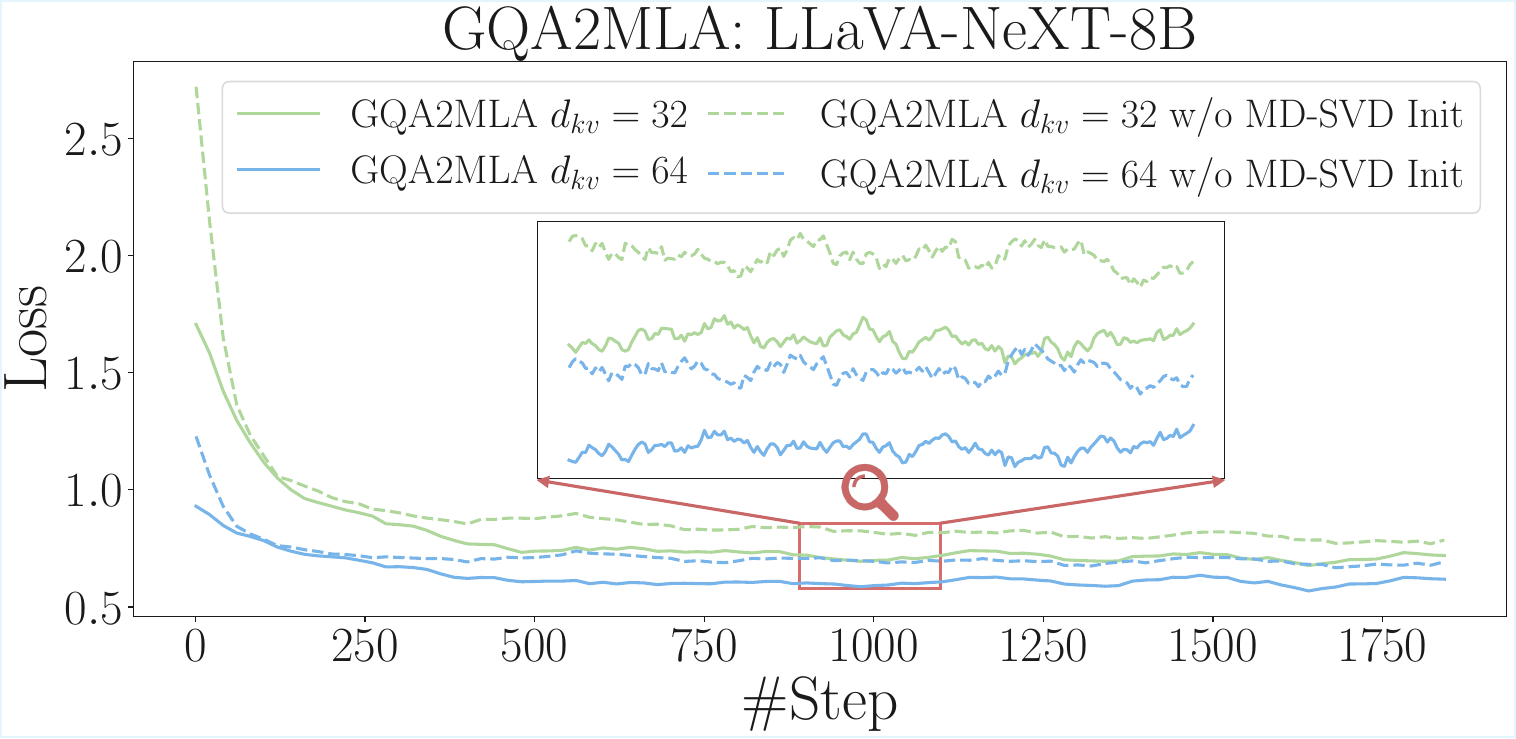}
    \caption{Training loss of GQA2MLA on LLaVA-NeXT w and w/o MD-SVD initialization under different $d_{kv}$.}
    \label{fig:woinit_loss}
\end{figure}

\paragraph{Effect of Modality-Decoupled SVD}

To better understand the contribution of our proposed MD-SVD, we disentangle its two core designs (e.g., modality decoupled and SVD-based initialization) and evaluate them separately. As depicted in Table~\ref{tab:ablation_study}, modality decoupled leads to consistent improvements across all $d_{kv}$ (+0.42 at $d_{kv}=32$, +0.45 at $d_{kv}=64$). While MD-SVD Init model exhibits slightly lower performance at $d_{kv}=\text{32}$, it achieves a substantial improvement at $d_{kv}=64$ (+1.23). Thus, SD-SVD Init performs better overall.

Moreover, Figure~\ref{fig:woinit_loss} demonstrates that under various KV cache compression conditions, our proposed MD-SVD initialization not only reduces the initial training loss, but also leads to better convergence throughout training. 
Overall, these results demonstrate that SVD Init enhances both optimization efficiency and model robustness across a range of low-ranks.

\paragraph{Effect of Two Stage Training} 
Table~\ref{tab:ablation_study} indicates that two-stage efficient parameter fine-tuning outperforms single-stage tuning at various cache compression rates, demonstrating the effectiveness of two-stage training.

\begin{table}[htb]
\centering
\small
\begin{tabular}{ccc}
  \toprule
  \textbf{Model} & \textbf{Partial-RoPE Strategy} & {\textbf{Avg.}} \\
  \midrule
  {$\text{LLaVA-NeXT}$} & $\mathcal{S}_{\text{2-norm}}$ & 70.31 \\
  \midrule
{$\text{LLaVA-NeXT}$} & $\mathcal{S}_{\text{MKL}}$ & \textbf{70.54} \\
  \bottomrule
\end{tabular}
\caption{Comparison of average performance under different Partial-RoPE strategies for LLaVA-NeXT.}
\label{tab:compare_2norm}
\end{table}
\paragraph{Effect of Partial-RoPE strategies}

Table~\ref{tab:compare_2norm} presents a comparison of LLaVA-NeXT performance under two different Partial-RoPE strategies. We observe that $\mathcal{S}_{\text{MKL}}$ achieves better performance than $\mathcal{S}_{\text{2-norm}}$. To better understand this behavior, we conduct a detailed analysis in~\Cref{sec:2norm_analysis}.

\section{Analysis}
\subsection{Empirical Validation for MD-SVD}

\label{sec:svd_proof}
Theorem~\ref{throrem:md_svd} proves that the loss from our proposed Modality-Decoupled SVD strategy is always smaller than the loss from jointly optimizing over visual and text (e.g. $\min\mathcal{L}^{2}_{joint}\ge \min\mathcal{L}^{2}_{visual}+\min\mathcal{L}^{2}_{text}$). To empirically validate this theoretical insight, we compute the ratio between split and joint losses, defined as ${\mathcal{L}^{2}_{split}}/{\mathcal{L}^{2}_{joint}}$ to and report this across all layers for different VLM architectures, as shown in Figure~\ref{fig:svdllm_loss}.

First, across all models, the ratio curves (three solid lines) consistently fall below the joint baseline (red dotted line) throughout all layers. This observation indicates that processing visual and text modalities decoupled leads to a strictly lower loss compared to joint processing, irrespective of architecture, which empirically validates Theorem 1.

Second, the relative advantage of the decoupled strategy becomes increasingly significant as the layer increases. It indicates that deeper layers are more sensitive to cross-modal interference and benefit more from modality-decoupled optimization. For example, in Qwen2.5-VL, the ratio decreases progressively and reaches a maximum relative loss reduction of 35.85\% at the final layer, highlighting the decoupled approach's compounding benefit in deep architectures.

\begin{figure}[htb]
    \centering
    \includegraphics[width=1.0\linewidth]{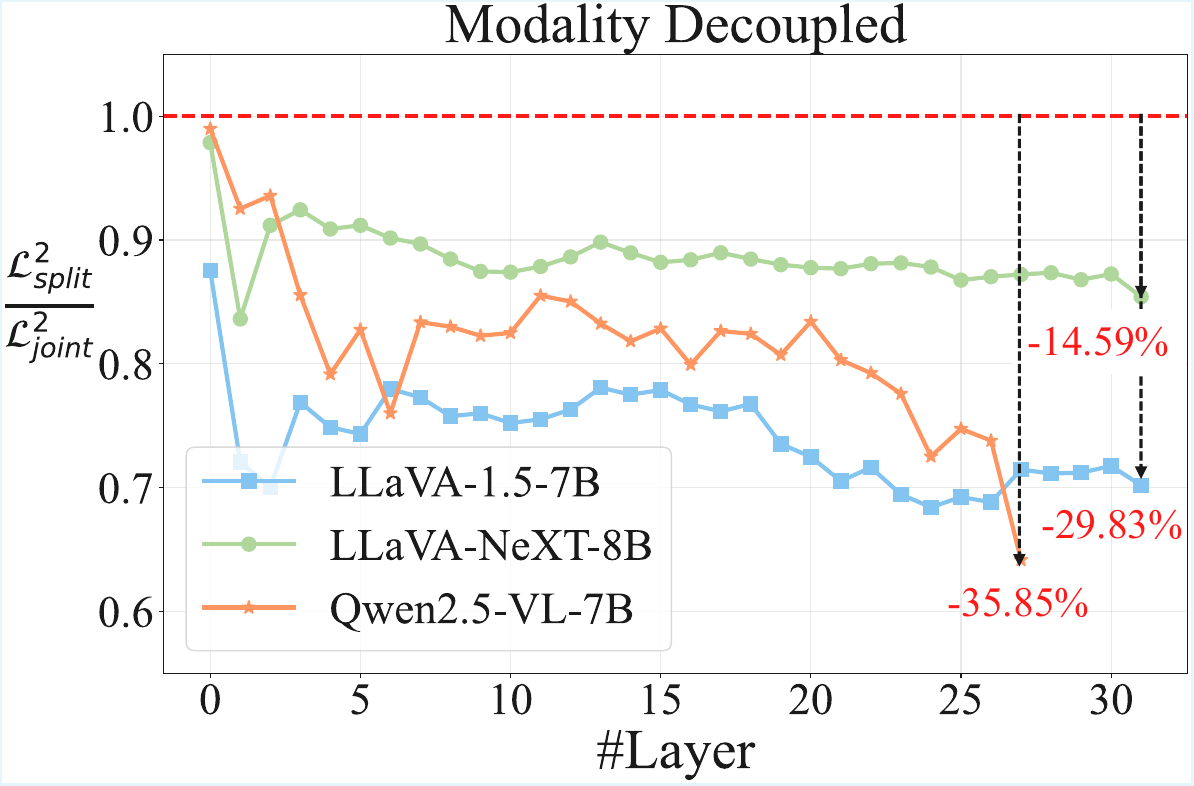}
    \caption{Quantitative analysis of MD-SVD via layer-wise loss ratio. Modality decoupled shows consistent improvements over joint optimization across all models.}
    \label{fig:svdllm_loss}
\end{figure}

\subsection{Comparison of  $\mathcal{S}_{\text{2-norm}}$ and $\mathcal{S}_{\text{MKL}}$}
\label{sec:2norm_analysis}

~\Cref{fig:rank_version} analyze the frequency profiles of the selected key dimensions under two strategies to understand their frequency preferences. $\mathcal{S}_{\text{MKL}}$ consistently selects dimensions that concentrate more heavily in the high-frequency range, as indicated by medians and tighter interquartile ranges.
In contrast, $\mathcal{S}_{\text{2-norm}}$ exhibits a broader spread and a strong tendency toward low-frequency selection, which are always noisy and less informative. 
$S_{\text{MKL}}$ explicitly evaluates the impact of each frequency component on the final attention distribution. Frequencies that induce larger changes are considered more informative and thus preferentially selected, allowing the model to retain discriminative high-frequency dimensions and improving the effectiveness of key-value compression.

Empirical results show that this sensitivity-based criterion  consistently selects more important key dimensions, and leads to lower attention loss and better downstream performance, especially in multimodal settings where fine-grained alignment is crucial.

\section{Related Work}
\label{sec:related-work}

\paragraph{Vision-Language Models} Represented by GPT4
~\cite{openai2024gpt4technicalreport}, VLMs have shown their strong strength and are increasingly becoming one of the mainstream research directions. 
They combine visual and language models to achieve cross-modal understanding and reasoning capabilities. 
Pioneering models such as 
LLaVA~\cite{DBLP:conf/cvpr/LiuLLL24} uses a simple projection layer to promote image-text alignment and uses a two-stage training method to improve model capabilities. 
Furthermore, MouSi~\cite{fan2024mousipolyvisualexpertvisionlanguagemodels} and Cambrian-1~\cite{DBLP:conf/nips/TongBWWIAYYMWPF24} leverage the unique attributes of diverse visual encoders and unify their strengths to enrich the multimodal understanding of VLMs.
Recently, the InternLM-XComposer~\cite{zhang2023internlmxcomposervisionlanguagelargemodel,dong2024internlmxcomposer2masteringfreeformtextimage,zhang2024internlmxcomposer25versatilelargevision} and InternVL~\cite{chen2024internvlscalingvisionfoundation,DBLP:journals/chinaf/ChenWTYGCTHLMMWDYGHSJXW24,zhu2025internvl3exploringadvancedtraining} family of models have shown leading performance. 

Moreover, videos or multiple images require more tokens for visual signals. For example, VideoPoet~\cite{DBLP:conf/icml/KondratyukYGLHS24} and VideoLLaVA~\cite{DBLP:conf/emnlp/LinYZCNJ024} encode each frame with thousands of tokens, quickly saturating computation budgets. This token explosion makes attention the primary bottleneck and calls for stronger sparsification strategies to unlock the next level of VLM performance. 

\paragraph{Efficient Architectures}  

KV cache consumption is a significant challenge for LLMs and VLMs, especially when dealing with long contexts. 
The primary methods for KV-cache compression can be broadly categorized into three approaches:

Model-level optimization like MQA~\cite{shazeer2019fasttransformerdecodingwritehead} and GQA~\cite{DBLP:conf/emnlp/AinslieLJZLS23} use intra-layer grouping, sharing key-value pairs across heads to save memory. MLA~\cite{deepseekai2024deepseekv2strongeconomicalefficient} further compresses KV caches via a low-rank joint representation across heads and layers, achieving substantial savings and demonstrating effectiveness in large-scale deployments.

Cache quantization techniques, such as KVQuant~\cite{DBLP:conf/nips/HooperKMMSKG24}, MassiveActivation~\cite{sun2024massiveactivationslargelanguage}, and HQQ~\cite{badri2023hqq}, compress KV cache into low-bit formats (e.g., 4-bit or 2-bit), significantly reducing memory footprint and improving inference efficiency. However, they may suffer from performance degradation under aggressive quantization or cross-modal inputs.

Cache pruning methods, including H$_2$O~\cite{DBLP:conf/nips/Zhang00CZC0TRBW23}, TOVA~\cite{DBLP:conf/emnlp/OrenHNA024}, PyramidKV~\cite{cai2025pyramidkvdynamickvcache}, and AdaKV~\cite{feng2025adakvoptimizingkvcache}, aim to reduce KV cache size by eliminating less informative tokens or attention heads, but may mistakenly discard contextually crucial information, especially under long-range dependencies.

\begin{figure}[t]
    \centering
    \includegraphics[width=1.0\linewidth]{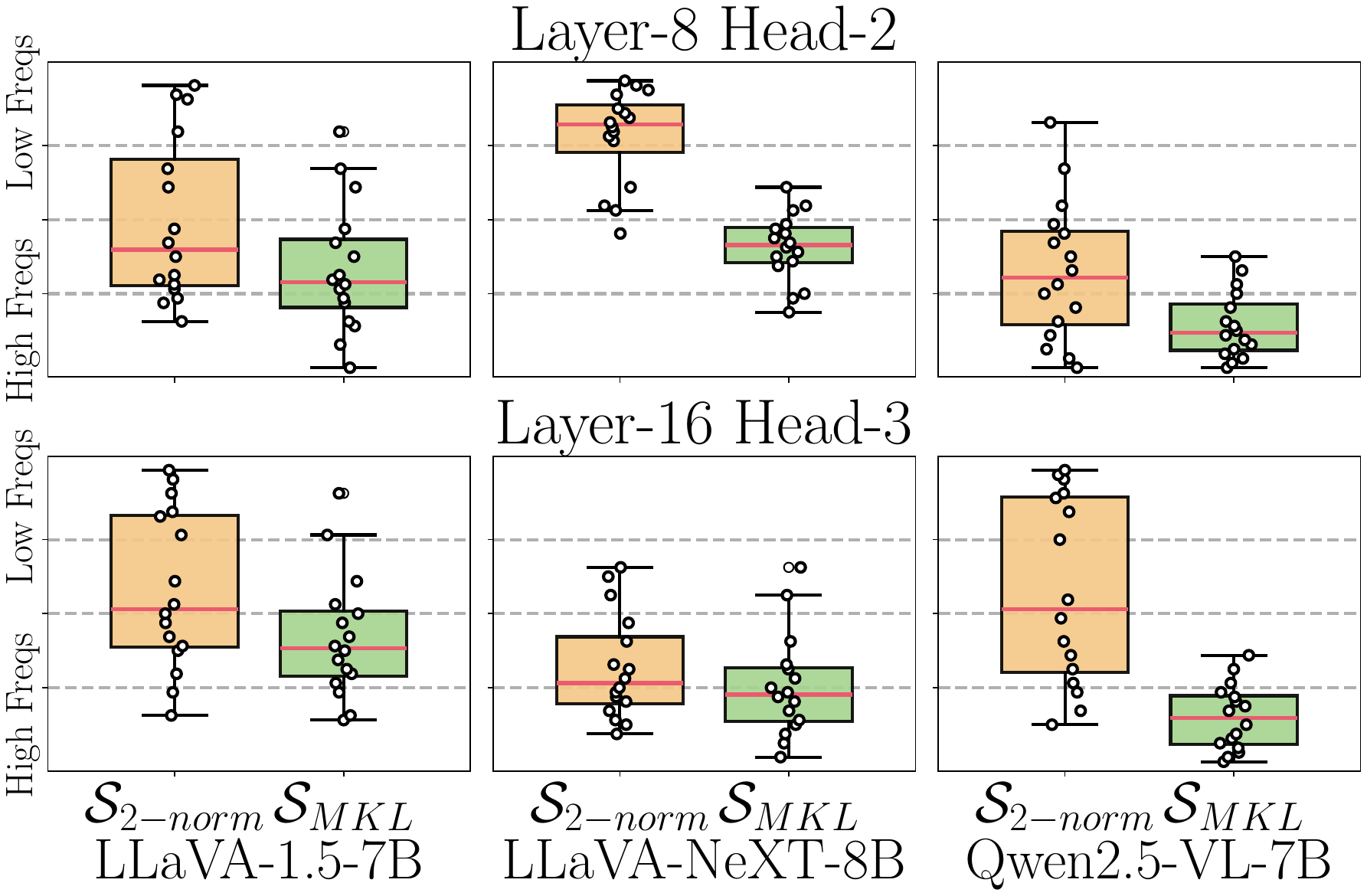}
    \caption{Comparison of multimodal partial-rope selection between $\mathcal{S}_{\text{2-norm}}$ and $\mathcal{S}_{\text{MKL}}$.}
    \label{fig:rank_version}
    
\end{figure}
\section{Conclusion}

In this work, we present \textbf{MHA2MLA-VLM}, a parameter-efficient and multimodal-aware framework for adapting VLMs to DeepSeek's MLA architecture. 
Our approach introduces modality-adaptive partial-RoPE and modality-decoupled low-rank approximation, enabling substantial KV cache reduction with minimal fine-tuning across various models.
Extensive experiments demonstrate that MHA2MLA-VLM achieves efficient inference while maintaining original model performance, offering a practical solution for scalable multimodal applications.

\section*{Acknowledgements}
The authors wish to thank the anonymous reviewers for their helpful comments. 
This work was partially funded by the Major Key Project of PCL under Grant PCL2024A06, National Natural Science Foundation of China (No.62576106, 62376061, 62476061, 62506079), Shanghai Rising-Star Program (23QA1400200), Natural Science Foundation of Shanghai (23ZR1403500) and Fudan Kunpeng \& Ascend Center of Cultivation. The computations in this research were performed using the CFFF platform of Fudan University.
\bibliography{aaai2026}

@misc{c:22,
      title={Attention Is All You Need}, 
      author={Ashish Vaswani and Noam Shazeer and Niki Parmar and Jakob Uszkoreit and Llion Jones and Aidan N. Gomez and Lukasz Kaiser and Illia Polosukhin},
      year={2017},
      eprint={1706.03762},
      archivePrefix={arXiv},
      primaryClass={cs.CL}
}

@article{DBLP:journals/ijon/SuALPBL24,
  author       = {Jianlin Su and
                  Murtadha H. M. Ahmed and
                  Yu Lu and
                  Shengfeng Pan and
                  Wen Bo and
                  Yunfeng Liu},
  title        = {RoFormer: Enhanced transformer with Rotary Position Embedding},
  journal      = {Neurocomputing},
  volume       = {568},
  pages        = {127063},
  year         = {2024},
  url          = {https://doi.org/10.1016/j.neucom.2023.127063},
  doi          = {10.1016/J.NEUCOM.2023.127063},
  bibsource    = {dblp computer science bibliography, https://dblp.org}
}

@misc{wang2024qwen2vlenhancingvisionlanguagemodels,
      title={Qwen2-VL: Enhancing Vision-Language Model's Perception of the World at Any Resolution}, 
      author={Peng Wang and Shuai Bai and Sinan Tan and et al.},
      year={2024},
      eprint={2409.12191},
      archivePrefix={arXiv},
      primaryClass={cs.CV},
      url={https://arxiv.org/abs/2409.12191}, 
}

@misc{bai2025qwen25vltechnicalreport,
      title={Qwen2.5-VL Technical Report}, 
      author={Shuai Bai and Keqin Chen and Xuejing Liu and et al.},
      year={2025},
      eprint={2502.13923},
      archivePrefix={arXiv},
      primaryClass={cs.CV},
      url={https://arxiv.org/abs/2502.13923}, 
}

@inproceedings{DBLP:conf/cvpr/LiuLLL24,
  author       = {Haotian Liu and
                  Chunyuan Li and
                  Yuheng Li and
                  Yong Jae Lee},
  title        = {Improved Baselines with Visual Instruction Tuning},
  booktitle    = {CVPR 2024},
  pages        = {26286--26296},
  publisher    = {{IEEE}},
  year         = {2024},
  url          = {https://doi.org/10.1109/CVPR52733.2024.02484},
  doi          = {10.1109/CVPR52733.2024.02484},
  bibsource    = {dblp computer science bibliography, https://dblp.org}
}

@misc{bai2023qwenvlversatilevisionlanguagemodel,
      title={Qwen-VL: A Versatile Vision-Language Model for Understanding, Localization, Text Reading, and Beyond}, 
      author={Jinze Bai and Shuai Bai and Shusheng Yang and Shijie Wang and Sinan Tan and Peng Wang and Junyang Lin and Chang Zhou and Jingren Zhou},
      year={2023},
      eprint={2308.12966},
      archivePrefix={arXiv},
      primaryClass={cs.CV},
      url={https://arxiv.org/abs/2308.12966}, 
}

@inproceedings{DBLP:conf/acl/JiGWGSCQZG25,
  author       = {Tao Ji and
                  Bin Guo and
                  Yuanbin Wu and
                  Qipeng Guo and
                  Shenlixing Shenlixing and
                  Chenzhan Chenzhan and
                  Xipeng Qiu and
                  Qi Zhang and
                  Tao Gui},
  editor       = {Wanxiang Che and
                  Joyce Nabende and
                  Ekaterina Shutova and
                  Mohammad Taher Pilehvar},
  title        = {Towards Economical Inference: Enabling DeepSeek's Multi-Head Latent
                  Attention in Any Transformer-based LLMs},
  booktitle    = {ACL 2025},
  pages        = {33313--33328},
  publisher    = {Association for Computational Linguistics},
  year         = {2025},
  url          = {https://aclanthology.org/2025.acl-long.1597/},
  bibsource    = {dblp computer science bibliography, https://dblp.org}
}

@misc{gpt-neo,
  author       = {Black, Sid and
                  Gao, Leo and
                  Wang, Phil and
                  Leahy, Connor and
                  Biderman, Stella},
  title        = {{GPT-Neo: Large Scale Autoregressive Language 
                   Modeling with Mesh-Tensorflow}},
  month        = mar,
  year         = 2021,
  note         = {{If you use this software, please cite it using 
                   these metadata.}},
  publisher    = {Zenodo},
  version      = {1.0},
  doi          = {10.5281/zenodo.5297715},
  url          = {https://doi.org/10.5281/zenodo.5297715}
}

@misc{barbero2025roundroundgomakes,
      title={Round and Round We Go! What makes Rotary Positional Encodings useful?}, 
      author={Federico Barbero and Alex Vitvitskyi and Christos Perivolaropoulos and Razvan Pascanu and Petar Veličković},
      year={2025},
      eprint={2410.06205},
      archivePrefix={arXiv},
      primaryClass={cs.CL},
      url={https://arxiv.org/abs/2410.06205}, 
}

@misc{xu2023parameterefficientfinetuningmethodspretrained,
      title={Parameter-Efficient Fine-Tuning Methods for Pretrained Language Models: A Critical Review and Assessment}, 
      author={Lingling Xu and Haoran Xie and Si-Zhao Joe Qin and Xiaohui Tao and Fu Lee Wang},
      year={2023},
      eprint={2312.12148},
      archivePrefix={arXiv},
      primaryClass={cs.CL},
      url={https://arxiv.org/abs/2312.12148}, 
}

@inproceedings{DBLP:conf/naacl/WangAWSZ25,
  author       = {Xin Wang and
                  Samiul Alam and
                  Zhongwei Wan and
                  Hui Shen and
                  Mi Zhang},
  editor       = {Luis Chiruzzo and
                  Alan Ritter and
                  Lu Wang},
  title        = {{SVD-LLM} {V2:} Optimizing Singular Value Truncation for Large Language
                  Model Compression},
  booktitle    = {NAACL 2025},
  pages        = {4287--4296},
  publisher    = {Association for Computational Linguistics},
  year         = {2025},
  url          = {https://doi.org/10.18653/v1/2025.naacl-long.217},
  doi          = {10.18653/V1/2025.NAACL-LONG.217},
  bibsource    = {dblp computer science bibliography, https://dblp.org}
}

@misc{wei2025videoropemakesgoodvideo,
      title={VideoRoPE: What Makes for Good Video Rotary Position Embedding?}, 
      author={Xilin Wei and Xiaoran Liu and Yuhang Zang and Xiaoyi Dong and Pan Zhang and Yuhang Cao and Jian Tong and Haodong Duan and Qipeng Guo and Jiaqi Wang and Xipeng Qiu and Dahua Lin},
      year={2025},
      eprint={2502.05173},
      archivePrefix={arXiv},
      primaryClass={cs.CV},
      url={https://arxiv.org/abs/2502.05173}, 
}

@inproceedings{DBLP:conf/iclr/HsuHCLSJ22,
  author       = {Yen{-}Chang Hsu and
                  Ting Hua and
                  Sungen Chang and
                  Qian Lou and
                  Yilin Shen and
                  Hongxia Jin},
  title        = {Language model compression with weighted low-rank factorization},
  booktitle    = {ICLR 2022},
  publisher    = {OpenReview.net},
  year         = {2022},
  url          = {https://openreview.net/forum?id=uPv9Y3gmAI5},
  bibsource    = {dblp computer science bibliography, https://dblp.org}
}

@inproceedings{DBLP:conf/iclr/01200W025,
  author       = {Xin Wang and
                  Yu Zheng and
                  Zhongwei Wan and
                  Mi Zhang},
  title        = {{SVD-LLM:} Truncation-aware Singular Value Decomposition for Large
                  Language Model Compression},
  booktitle    = {ICLR 2025},
  publisher    = {OpenReview.net},
  year         = {2025},
  url          = {https://openreview.net/forum?id=LNYIUouhdt},
  bibsource    = {dblp computer science bibliography, https://dblp.org}
}

@inproceedings{DBLP:conf/eccv/KembhaviSKSHF16,
  author       = {Aniruddha Kembhavi and
                  Mike Salvato and
                  Eric Kolve and
                  Min Joon Seo and
                  Hannaneh Hajishirzi and
                  Ali Farhadi},
  editor       = {Bastian Leibe and
                  Jiri Matas and
                  Nicu Sebe and
                  Max Welling},
  title        = {A Diagram is Worth a Dozen Images},
  booktitle    = {Proc. ECCV 2016},
  series       = {Lecture Notes in Computer Science},
  volume       = {9908},
  pages        = {235--251},
  publisher    = {Springer},
  year         = {2016},
  url          = {https://doi.org/10.1007/978-3-319-46493-0\_15},
  doi          = {10.1007/978-3-319-46493-0\_15},
  bibsource    = {dblp computer science bibliography, https://dblp.org}
}

@inproceedings{emnlp/AinslieLJZLS23,
  author       = {Joshua Ainslie and
                  James Lee{-}Thorp and
                  Michiel de Jong and
                  Yury Zemlyanskiy and
                  Federico Lebr{\'{o}}n and
                  Sumit Sanghai},
  editor       = {Houda Bouamor and
                  Juan Pino and
                  Kalika Bali},
  title        = {{GQA:} Training Generalized Multi-Query Transformer Models from Multi-Head
                  Checkpoints},
  booktitle    = {Proc. EMNLP 2023},
  pages        = {4895--4901},
  publisher    = {Association for Computational Linguistics},
  year         = {2023},
  url          = {https://doi.org/10.18653/v1/2023.emnlp-main.298},
  doi          = {10.18653/V1/2023.EMNLP-MAIN.298},
  bibsource    = {dblp computer science bibliography, https://dblp.org}
}

@misc{deepseekai2024deepseekv2strongeconomicalefficient,
      title={DeepSeek-V2: A Strong, Economical, and Efficient Mixture-of-Experts Language Model}, 
      author={DeepSeek-AI and Aixin Liu and Bei Feng and Bin Wang and et al.},
      year={2024},
      eprint={2405.04434},
      archivePrefix={arXiv},
      primaryClass={cs.CL},
      url={https://arxiv.org/abs/2405.04434}, 
}

@article{DBLP:journals/tmlr/0002LTTXCHD0025,
  author       = {Haoyang Li and
                  Yiming Li and
                  Anxin Tian and
                  Tianhao Tang and
                  Zhanchao Xu and
                  Xuejia Chen and
                  Nicole Hu and
                  Wei Dong and
                  Qing Li and
                  Lei Chen},
  title        = {A Survey on Large Language Model Acceleration based on {KV} Cache
                  Management},
  journal      = {Trans. Mach. Learn. Res.},
  volume       = {2025},
  year         = {2025},
  url          = {https://openreview.net/forum?id=z3JZzu9EA3},
  bibsource    = {dblp computer science bibliography, https://dblp.org}
}

@inproceedings{DBLP:conf/alt/KelesWH23,
  author       = {Feyza Duman Keles and
                  Pruthuvi Mahesakya Wijewardena and
                  Chinmay Hegde},
  editor       = {Shipra Agrawal and
                  Francesco Orabona},
  title        = {On The Computational Complexity of Self-Attention},
  booktitle    = {ALT 2023},
  series       = {Proceedings of Machine Learning Research},
  volume       = {201},
  pages        = {597--619},
  publisher    = {{PMLR}},
  year         = {2023},
  url          = {https://proceedings.mlr.press/v201/duman-keles23a.html},
  bibsource    = {dblp computer science bibliography, https://dblp.org}
}

@misc{pan2025surveyslowthinkingbasedreasoning,
      title={A Survey of Slow Thinking-based Reasoning LLMs using Reinforced Learning and Inference-time Scaling Law}, 
      author={Qianjun Pan and Wenkai Ji and Yuyang Ding and Junsong Li and Shilian Chen and Junyi Wang and Jie Zhou and Qin Chen and Min Zhang and Yulan Wu and Liang He},
      year={2025},
      eprint={2505.02665},
      archivePrefix={arXiv},
      primaryClass={cs.AI},
      url={https://arxiv.org/abs/2505.02665}, 
}

@misc{bordes2024introductionvisionlanguagemodeling,
      title={An Introduction to Vision-Language Modeling}, 
      author={Florian Bordes and Richard Yuanzhe Pang and Anurag Ajay and et al.},
      year={2024},
      eprint={2405.17247},
      archivePrefix={arXiv},
      primaryClass={cs.LG},
      url={https://arxiv.org/abs/2405.17247}, 
}

@inproceedings{DBLP:conf/cvpr/HudsonM19,
  author       = {Drew A. Hudson and
                  Christopher D. Manning},
  title        = {{GQA:} {A} New Dataset for Real-World Visual Reasoning and Compositional
                  Question Answering},
  booktitle    = {CVPR 2019},
  pages        = {6700--6709},
  publisher    = {Computer Vision Foundation / {IEEE}},
  year         = {2019},
  url          = {http://openaccess.thecvf.com/content\_CVPR\_2019/html/Hudson\_GQA\_A\_New\_Dataset\_for\_Real-World\_Visual\_Reasoning\_and\_Compositional\_CVPR\_2019\_paper.html},
  doi          = {10.1109/CVPR.2019.00686},
  bibsource    = {dblp computer science bibliography, https://dblp.org}
}

@inproceedings{DBLP:conf/emnlp/LiDZWZW23,
  author       = {Yifan Li and
                  Yifan Du and
                  Kun Zhou and
                  Jinpeng Wang and
                  Wayne Xin Zhao and
                  Ji{-}Rong Wen},
  editor       = {Houda Bouamor and
                  Juan Pino and
                  Kalika Bali},
  title        = {Evaluating Object Hallucination in Large Vision-Language Models},
  booktitle    = {EMNLP 2023},
  pages        = {292--305},
  publisher    = {Association for Computational Linguistics},
  year         = {2023},
  url          = {https://doi.org/10.18653/v1/2023.emnlp-main.20},
  doi          = {10.18653/V1/2023.EMNLP-MAIN.20},
  bibsource    = {dblp computer science bibliography, https://dblp.org}
}

@misc{li2023seedbenchbenchmarkingmultimodalllms,
      title={SEED-Bench: Benchmarking Multimodal LLMs with Generative Comprehension}, 
      author={Bohao Li and Rui Wang and Guangzhi Wang and Yuying Ge and Yixiao Ge and Ying Shan},
      year={2023},
      eprint={2307.16125},
      archivePrefix={arXiv},
      primaryClass={cs.CL},
      url={https://arxiv.org/abs/2307.16125}, 
}

@inproceedings{DBLP:conf/iclr/0004ZTFZWLWW00025,
  author       = {Yifan Zhang and
                  Huanyu Zhang and
                  Haochen Tian and
                  Chaoyou Fu and
                  Shuangqing Zhang and
                  Junfei Wu and
                  Feng Li and
                  Kun Wang and
                  Qingsong Wen and
                  Zhang Zhang and
                  Liang Wang and
                  Rong Jin},
  title        = {MME-RealWorld: Could Your Multimodal {LLM} Challenge High-Resolution
                  Real-World Scenarios that are Difficult for Humans?},
  booktitle    = {ICLR 2025},
  publisher    = {OpenReview.net},
  year         = {2025},
  url          = {https://openreview.net/forum?id=k5VHHgsRbi},
  bibsource    = {dblp computer science bibliography, https://dblp.org}
}

@inproceedings{DBLP:conf/eccv/LiuDZLZZYWHLCL24,
  author       = {Yuan Liu and
                  Haodong Duan and
                  Yuanhan Zhang and
                  Bo Li and
                  Songyang Zhang and
                  Wangbo Zhao and
                  Yike Yuan and
                  Jiaqi Wang and
                  Conghui He and
                  Ziwei Liu and
                  Kai Chen and
                  Dahua Lin},
  editor       = {Ales Leonardis and
                  Elisa Ricci and
                  Stefan Roth and
                  Olga Russakovsky and
                  Torsten Sattler and
                  G{\"{u}}l Varol},
  title        = {MMBench: Is Your Multi-modal Model an All-Around Player?},
  booktitle    = {ECCV 2024},
  series       = {Lecture Notes in Computer Science},
  volume       = {15064},
  pages        = {216--233},
  publisher    = {Springer},
  year         = {2024},
  url          = {https://doi.org/10.1007/978-3-031-72658-3\_13},
  doi          = {10.1007/978-3-031-72658-3\_13},
  bibsource    = {dblp computer science bibliography, https://dblp.org}
}

@inproceedings{DBLP:conf/acl/MasryLTJH22,
  author       = {Ahmed Masry and
                  Do Xuan Long and
                  Jia Qing Tan and
                  Shafiq R. Joty and
                  Enamul Hoque},
  editor       = {Smaranda Muresan and
                  Preslav Nakov and
                  Aline Villavicencio},
  title        = {ChartQA: {A} Benchmark for Question Answering about Charts with Visual
                  and Logical Reasoning},
  booktitle    = {Findings of ACL 2022},
  pages        = {2263--2279},
  publisher    = {Association for Computational Linguistics},
  year         = {2022},
  url          = {https://doi.org/10.18653/v1/2022.findings-acl.177},
  doi          = {10.18653/V1/2022.FINDINGS-ACL.177},
  bibsource    = {dblp computer science bibliography, https://dblp.org}
}

@inproceedings{DBLP:conf/wacv/MathewKJ21,
  author       = {Minesh Mathew and
                  Dimosthenis Karatzas and
                  C. V. Jawahar},
  title        = {DocVQA: {A} Dataset for {VQA} on Document Images},
  booktitle    = {WACV 2021},
  pages        = {2199--2208},
  publisher    = {{IEEE}},
  year         = {2021},
  url          = {https://doi.org/10.1109/WACV48630.2021.00225},
  doi          = {10.1109/WACV48630.2021.00225},
  bibsource    = {dblp computer science bibliography, https://dblp.org}
}

@misc{fan2024mousipolyvisualexpertvisionlanguagemodels,
      title={MouSi: Poly-Visual-Expert Vision-Language Models}, 
      author={Xiaoran Fan and Tao Ji and Changhao Jiang and Shuo Li and Senjie Jin and Sirui Song and Junke Wang and Boyang Hong and Lu Chen and Guodong Zheng and Ming Zhang and Caishuang Huang and Rui Zheng and Zhiheng Xi and Yuhao Zhou and Shihan Dou and Junjie Ye and Hang Yan and Tao Gui and Qi Zhang and Xipeng Qiu and Xuanjing Huang and Zuxuan Wu and Yu-Gang Jiang},
      year={2024},
      eprint={2401.17221},
      archivePrefix={arXiv},
      primaryClass={cs.CV},
      url={https://arxiv.org/abs/2401.17221}, 
}

@misc{liu2024llavanext,
    title={LLaVA-NeXT: Improved reasoning, OCR, and world knowledge},
    url={https://llava-vl.github.io/blog/2024-01-30-llava-next/},
    author={Liu, Haotian and Li, Chunyuan and Li, Yuheng and Li, Bo and Zhang, Yuanhan and Shen, Sheng and Lee, Yong Jae},
    month={January},
    year={2024}
}

@misc{openai2024gpt4technicalreport,
      title={GPT-4 Technical Report}, 
      author={OpenAI and Josh Achiam and Steven Adler and Sandhini Agarwal and et al.},
      year={2024},
      eprint={2303.08774},
      archivePrefix={arXiv},
      primaryClass={cs.CL},
      url={https://arxiv.org/abs/2303.08774}, 
}

@inproceedings{DBLP:conf/nips/TongBWWIAYYMWPF24,
  author       = {Peter Tong and
                  Ellis Brown and
                  Penghao Wu and
                  Sanghyun Woo and
                  Adithya Iyer and
                  Sai Charitha Akula and
                  Shusheng Yang and
                  Jihan Yang and
                  Manoj Middepogu and
                  Ziteng Wang and
                  Xichen Pan and
                  Rob Fergus and
                  Yann LeCun and
                  Saining Xie},
  editor       = {Amir Globersons and
                  Lester Mackey and
                  Danielle Belgrave and
                  Angela Fan and
                  Ulrich Paquet and
                  Jakub M. Tomczak and
                  Cheng Zhang},
  title        = {Cambrian-1: {A} Fully Open, Vision-Centric Exploration of Multimodal
                  LLMs},
  booktitle    = {NeurIPS 2024},
  year         = {2024},
  url          = {http://papers.nips.cc/paper\_files/paper/2024/hash/9ee3a664ccfeabc0da16ac6f1f1cfe59-Abstract-Conference.html},
  bibsource    = {dblp computer science bibliography, https://dblp.org}
}

@misc{zhang2023internlmxcomposervisionlanguagelargemodel,
      title={InternLM-XComposer: A Vision-Language Large Model for Advanced Text-image Comprehension and Composition}, 
      author={Pan Zhang and Xiaoyi Dong and Bin Wang and et al.},
      year={2023},
      eprint={2309.15112},
      archivePrefix={arXiv},
      primaryClass={cs.CV},
      url={https://arxiv.org/abs/2309.15112}, 
}

@misc{dong2024internlmxcomposer2masteringfreeformtextimage,
      title={InternLM-XComposer2: Mastering Free-form Text-Image Composition and Comprehension in Vision-Language Large Model}, 
      author={Xiaoyi Dong and Pan Zhang and Yuhang Zang and et al.},
      year={2024},
      eprint={2401.16420},
      archivePrefix={arXiv},
      primaryClass={cs.CV},
      url={https://arxiv.org/abs/2401.16420}, 
}

@misc{zhang2024internlmxcomposer25versatilelargevision,
      title={InternLM-XComposer-2.5: A Versatile Large Vision Language Model Supporting Long-Contextual Input and Output}, 
      author={Pan Zhang and Xiaoyi Dong and Yuhang Zang and et al.},
      year={2024},
      eprint={2407.03320},
      archivePrefix={arXiv},
      primaryClass={cs.CV},
      url={https://arxiv.org/abs/2407.03320}, 
}

@misc{chen2024internvlscalingvisionfoundation,
      title={InternVL: Scaling up Vision Foundation Models and Aligning for Generic Visual-Linguistic Tasks}, 
      author={Zhe Chen and Jiannan Wu and Wenhai Wang and et al.},
      year={2024},
      eprint={2312.14238},
      archivePrefix={arXiv},
      primaryClass={cs.CV},
      url={https://arxiv.org/abs/2312.14238}, 
}

@article{DBLP:journals/chinaf/ChenWTYGCTHLMMWDYGHSJXW24,
  author       = {Zhe Chen and
                  Weiyun Wang and
                  Hao Tian and
                  Shenglong Ye and
                  Zhangwei Gao and
                  Erfei Cui and
                  Wenwen Tong and
                  Kongzhi Hu and
                  Jiapeng Luo and
                  Zheng Ma and
                  Ji Ma and
                  Jiaqi Wang and
                  Xiaoyi Dong and
                  Hang Yan and
                  Hewei Guo and
                  Conghui He and
                  Botian Shi and
                  Zhenjiang Jin and
                  Chao Xu and
                  Bin Wang and
                  Xingjian Wei and
                  Wei Li and
                  Wenjian Zhang and
                  Bo Zhang and
                  Pinlong Cai and
                  Licheng Wen and
                  Xiangchao Yan and
                  Min Dou and
                  Lewei Lu and
                  Xizhou Zhu and
                  Tong Lu and
                  Dahua Lin and
                  Yu Qiao and
                  Jifeng Dai and
                  Wenhai Wang},
  title        = {How far are we to GPT-4V? Closing the gap to commercial multimodal
                  models with open-source suites},
  journal      = {Sci. China Inf. Sci.},
  volume       = {67},
  number       = {12},
  year         = {2024},
  url          = {https://doi.org/10.1007/s11432-024-4231-5},
  doi          = {10.1007/S11432-024-4231-5},
  bibsource    = {dblp computer science bibliography, https://dblp.org}
}

@misc{zhu2025internvl3exploringadvancedtraining,
      title={InternVL3: Exploring Advanced Training and Test-Time Recipes for Open-Source Multimodal Models}, 
      author={Jinguo Zhu and Weiyun Wang and Zhe Chen and et al.},
      year={2025},
      eprint={2504.10479},
      archivePrefix={arXiv},
      primaryClass={cs.CV},
      url={https://arxiv.org/abs/2504.10479}, 
}

@inproceedings{DBLP:conf/emnlp/LinYZCNJ024,
  author       = {Bin Lin and
                  Yang Ye and
                  Bin Zhu and
                  Jiaxi Cui and
                  Munan Ning and
                  Peng Jin and
                  Li Yuan},
  editor       = {Yaser Al{-}Onaizan and
                  Mohit Bansal and
                  Yun{-}Nung Chen},
  title        = {Video-LLaVA: Learning United Visual Representation by Alignment Before
                  Projection},
  booktitle    = {EMNLP 2024},
  pages        = {5971--5984},
  publisher    = {Association for Computational Linguistics},
  year         = {2024},
  url          = {https://doi.org/10.18653/v1/2024.emnlp-main.342},
  doi          = {10.18653/V1/2024.EMNLP-MAIN.342},
  bibsource    = {dblp computer science bibliography, https://dblp.org}
}

@inproceedings{DBLP:conf/icml/KondratyukYGLHS24,
  author       = {Dan Kondratyuk and
                  Lijun Yu and
                  Xiuye Gu and
                  Jos{\'{e}} Lezama and
                  Jonathan Huang and
                  Grant Schindler and
                  Rachel Hornung and
                  Vighnesh Birodkar and
                  Jimmy Yan and
                  Ming{-}Chang Chiu and
                  Krishna Somandepalli and
                  Hassan Akbari and
                  Yair Alon and
                  Yong Cheng and
                  Joshua V. Dillon and
                  Agrim Gupta and
                  Meera Hahn and
                  Anja Hauth and
                  David Hendon and
                  Alonso Martinez and
                  David Minnen and
                  Mikhail Sirotenko and
                  Kihyuk Sohn and
                  Xuan Yang and
                  Hartwig Adam and
                  Ming{-}Hsuan Yang and
                  Irfan Essa and
                  Huisheng Wang and
                  David A. Ross and
                  Bryan Seybold and
                  Lu Jiang},
  title        = {VideoPoet: {A} Large Language Model for Zero-Shot Video Generation},
  booktitle    = {ICML 2024},
  publisher    = {OpenReview.net},
  year         = {2024},
  url          = {https://openreview.net/forum?id=LRkJwPIDuE},
  bibsource    = {dblp computer science bibliography, https://dblp.org}
}

@inproceedings{DBLP:conf/emnlp/OrenHNA024,
  author       = {Matanel Oren and
                  Michael Hassid and
                  Yarden Nir and
                  Yossi Adi and
                  Roy Schwartz},
  editor       = {Yaser Al{-}Onaizan and
                  Mohit Bansal and
                  Yun{-}Nung Chen},
  title        = {Transformers are Multi-State RNNs},
  booktitle    = {EMNLP 2024},
  pages        = {18724--18741},
  publisher    = {Association for Computational Linguistics},
  year         = {2024},
  url          = {https://doi.org/10.18653/v1/2024.emnlp-main.1043},
  doi          = {10.18653/V1/2024.EMNLP-MAIN.1043},
  bibsource    = {dblp computer science bibliography, https://dblp.org}
}

@inproceedings{DBLP:conf/nips/Zhang00CZC0TRBW23,
  author       = {Zhenyu Zhang and
                  Ying Sheng and
                  Tianyi Zhou and
                  Tianlong Chen and
                  Lianmin Zheng and
                  Ruisi Cai and
                  Zhao Song and
                  Yuandong Tian and
                  Christopher R{\'{e}} and
                  Clark W. Barrett and
                  Zhangyang Wang and
                  Beidi Chen},
  editor       = {Alice Oh and
                  Tristan Naumann and
                  Amir Globerson and
                  Kate Saenko and
                  Moritz Hardt and
                  Sergey Levine},
  title        = {{H2O:} Heavy-Hitter Oracle for Efficient Generative Inference of Large
                  Language Models},
  booktitle    = {NeurIPS 2023},
  year         = {2023},
  url          = {http://papers.nips.cc/paper\_files/paper/2023/hash/6ceefa7b15572587b78ecfcebb2827f8-Abstract-Conference.html},
  bibsource    = {dblp computer science bibliography, https://dblp.org}
}

@misc{badri2023hqq,
  title  = {Half-Quadratic Quantization of Large Machine Learning Models},
  url    = {https://mobiusml.github.io/hqq_blog/},
  author = {Hicham Badri and Appu Shaji},
  month  = {November},
  year   = {2023}
}

@inproceedings{DBLP:conf/nips/HooperKMMSKG24,
  author       = {Coleman Hooper and
                  Sehoon Kim and
                  Hiva Mohammadzadeh and
                  Michael W. Mahoney and
                  Yakun Sophia Shao and
                  Kurt Keutzer and
                  Amir Gholami},
  editor       = {Amir Globersons and
                  Lester Mackey and
                  Danielle Belgrave and
                  Angela Fan and
                  Ulrich Paquet and
                  Jakub M. Tomczak and
                  Cheng Zhang},
  title        = {KVQuant: Towards 10 Million Context Length {LLM} Inference with {KV}
                  Cache Quantization},
  booktitle    = {NeurIPS 2024},
  year         = {2024},
  url          = {http://papers.nips.cc/paper\_files/paper/2024/hash/028fcbcf85435d39a40c4d61b42c99a4-Abstract-Conference.html},
  bibsource    = {dblp computer science bibliography, https://dblp.org}
}

@misc{sun2024massiveactivationslargelanguage,
      title={Massive Activations in Large Language Models}, 
      author={Mingjie Sun and Xinlei Chen and J. Zico Kolter and Zhuang Liu},
      year={2024},
      eprint={2402.17762},
      archivePrefix={arXiv},
      primaryClass={cs.CL},
      url={https://arxiv.org/abs/2402.17762}, 
}

@misc{cai2025pyramidkvdynamickvcache,
      title={PyramidKV: Dynamic KV Cache Compression based on Pyramidal Information Funneling}, 
      author={Zefan Cai and Yichi Zhang and Bofei Gao and Yuliang Liu and Yucheng Li and Tianyu Liu and Keming Lu and Wayne Xiong and Yue Dong and Junjie Hu and Wen Xiao},
      year={2025},
      eprint={2406.02069},
      archivePrefix={arXiv},
      primaryClass={cs.CL},
      url={https://arxiv.org/abs/2406.02069}, 
}

@misc{feng2025adakvoptimizingkvcache,
      title={Ada-KV: Optimizing KV Cache Eviction by Adaptive Budget Allocation for Efficient LLM Inference}, 
      author={Yuan Feng and Junlin Lv and Yukun Cao and Xike Xie and S. Kevin Zhou},
      year={2025},
      eprint={2407.11550},
      archivePrefix={arXiv},
      primaryClass={cs.CL},
      url={https://arxiv.org/abs/2407.11550}, 
}

@misc{shazeer2019fasttransformerdecodingwritehead,
      title={Fast Transformer Decoding: One Write-Head is All You Need}, 
      author={Noam Shazeer},
      year={2019},
      eprint={1911.02150},
      archivePrefix={arXiv},
      primaryClass={cs.NE},
      url={https://arxiv.org/abs/1911.02150}, 
}

@inproceedings{DBLP:conf/emnlp/AinslieLJZLS23,
  author       = {Joshua Ainslie and
                  James Lee{-}Thorp and
                  Michiel de Jong and
                  Yury Zemlyanskiy and
                  Federico Lebr{\'{o}}n and
                  Sumit Sanghai},
  editor       = {Houda Bouamor and
                  Juan Pino and
                  Kalika Bali},
  title        = {{GQA:} Training Generalized Multi-Query Transformer Models from Multi-Head
                  Checkpoints},
  booktitle    = {EMNLP 2023},
  pages        = {4895--4901},
  publisher    = {Association for Computational Linguistics},
  year         = {2023},
  url          = {https://doi.org/10.18653/v1/2023.emnlp-main.298},
  doi          = {10.18653/V1/2023.EMNLP-MAIN.298},
  bibsource    = {dblp computer science bibliography, https://dblp.org}
}

\clearpage
\appendix

\section{Appendix}
\label{sec:appendix}

\subsection{Experimental Setups}

\subsubsection{Models}

To validate our method's effectiveness across diverse VLM architectures, we evaluate three representative and widely-used models: LLaVA-1.5, LLaVA-NeXT, and Qwen2.5-VL-Instruct. These VLMs encompass key variations in attention mechanisms (MHA/GQA) and PE (vanilla RoPE and M-RoPE). 
Specifically,
LLaVA-1.5 adopts standard MHA, while LLaVA-NeXT and Qwen2.5-VL employ GQA (e.g., KV groups = 4 for LLaVA-NeXT and 7 for Qwen2.5-VL). Moreover, compared with LLaVA series VLMs (Vanilla RoPE), Qwen2.5-VL applies the M-RoPE.

For the two LLaVA variants, we adopt the community Hugging Face checkpoints (HF versions: \texttt{llava-1.5-7b-hf} and \texttt{llama3-llava-next-8b-hf}) rather than the original author releases, so that all baselines, including Qwen2.5-VL, share an identical training, inference, and evaluation pipeline. 
This choice offers greater compatibility and reproducibility within the Hugging Face ecosystem.

\subsubsection{Dataset} 
In order to implement MHA2MLA architecture migration in VLM, we use the data used to train the original model as much as possible.
We chose the LLaVA-series visual instruction fine-tuning datasets for LLaVA-1.5 Dataset and LLaVA-NeXT Dataset, because these instruction tuning data are open-source, which can minimize the gap in fine-tuning data and processes. 
We chose Qwen2.5-VL because it is one of the widely used open-source VLMs (but its pretraining and instruction tuning data are not open-source, there is a potential gap in finetuning).

Specifically, The LLaVA-1.5 dataset consists of approximately 665K samples, is constructed from vision-language pairs sourced from image captioning datasets, with high-quality multi-turn instructions generated to support open-ended vision-language tasks. 
The LLaVA-NeXT Dataset further expands the scale to 778K samples and incorporates more diverse and user-aligned instruction-following data.

In the fine-tuning of MHA2MLA-VLM, LLaVA-1.5 uses its own default data, while LLaVA-NeXT and Qwen2.5-VL are fine-tuned on the publicly released LLaVA-NeXT dataset. 
Specifically, For the LLaVA-1.5 and Qwen2.5-VL models, were trained on approximately 0.5B multimodal tokens, which contains image and text tokens.
LLaVA-NeXT uses a dynamic high-resolution strategy to achieve multiple images, the number of fine-tuned tokens is approximately 1.8B. 
In summary, compared to existing models training from scratch that rely on trillions of training tokens~\cite{bai2025qwen25vltechnicalreport}, our approach enables the migration from MHA to MLA architectures using only within the open-sourced instruction dataset. 

\subsubsection{PEFT Training Strategy}
To reduce the cost of MHA2MLA-VLM adaptation, we introduce parameter-efficient fine-tuning (PEFT), which contains two stages.
During the multimodal partial-rope phase (stage 1), only the two projection matrices for query and key are fine-tuned, while all other parameters are frozen. 
For the low-rank approximation phase (stage 2), only the parameters within MLA are fine-tuned. 
Specifically, for Qwen2.5-VL, our method only fine-tuning and approximately $\sim\! 6 \% $ and $\sim\! 10\% $ of the original model parameters in stage 1 and stage 2, respectively. It reduces the time required by 59\% (e.g., the MHA2MLA-VLM of Qwen2.5-VL is shortened from 22 hours to 9 hours). 



\subsubsection{HyperParameters}
As illustrated in Table~\ref{tab:s1_hp} and Table~\ref{tab:s2_hp}, we report the hyperparameters of all models we used for both stage 1 and stage 2 fine-tuning, with multimodal partial rope dim is $\frac{d_h}{4}$ as the default configuration.

\begin{table*}[t]
\centering
\small
\setlength\tabcolsep{3pt}
\begin{tabular}{l@{}lccc}
  \toprule
  \multicolumn{2}{l}{\textbf{Metrics}} & \textbf{LLaVA-1.5-7B} & \textbf{LLaVA-NeXT-8B} & \textbf{Qwen2.5-VL-7B} \\
  \midrule
  \multicolumn{2}{l}{n\_batch $\times$ n\_gpu} & 16$\times$8 &  16$\times$8 &  16$\times$8  \\
  \multicolumn{2}{l}{Learning Rate} & 5e-5 & 5e-6 & 1e-5 \\
  \multicolumn{2}{l}{Hardware} & NVIDIA A800 & NVIDIA A800 & NVIDIA A800 \\
  \multicolumn{2}{l}{Steps} & 5197 & 6084 & 6084 \\
  \multicolumn{2}{l}{Tokens} & 0.5B & 1.8B & 0.5B \\
  \multicolumn{2}{l}{Warmup ratio} & 10\% & 10\% & 10\% \\
  \multicolumn{2}{l}{Decay} & 10\% & 10\% & 10\% \\
  \multicolumn{2}{l}{Time} & 5.5 hours & 15 hours & 9 hours \\
  \multicolumn{2}{l}{Seqlen} & 2048 & 4096 & 4096 \\
  \multirow{1}{*}{\#Param.} & & 1497.64M (21.20\%) & 793.00M (9.49\%) & 494.48M(5.96\%) \\
  \arrayrulecolor{black}
  \bottomrule
\end{tabular}
\caption{Training configurations and model Parameter summary across model for stage 1.}
\label{tab:s1_hp}
\end{table*}

\begin{table*}[t]
\centering
\small
\setlength\tabcolsep{3pt}
\begin{tabular}{l@{}lccc}
  \toprule
  \multicolumn{2}{l}{\textbf{Metrics}} & \textbf{LLaVA-1.5-7B} & \textbf{LLaVA-NeXT-8B} & \textbf{Qwen2.5-VL-7B} \\
  \midrule
  \multicolumn{2}{l}{n\_batch $\times$ n\_gpu} &  16$\times$8 &  16$\times$8 &  16$\times$8  \\
  \multirow{4}{*}{Learning Rate} & $d_{kv}=16$ & 1e-5 & - & - \\
  & $d_{kv}=32$ & 5e-6 & 5e-6 & 5e-5 \\
  & $d_{kv}=64$ & 1e-6 & 1e-6 & 1e-5 \\ 
  & $d_{kv}=128$ & - & 8e-7 & 5e-6 \\
  \multicolumn{2}{l}{Hardware} & NVIDIA A800 & NVIDIA A800 & NVIDIA A800 \\
  \multicolumn{2}{l}{Steps} & 5197 & 6084 & 6084 \\
  \multicolumn{2}{l}{Tokens} & 0.5B & 1.8B & 0.5B \\
  \multicolumn{2}{l}{Warmup ratio} & 10\% & 10\% & 10\% \\
  \multicolumn{2}{l}{Decay} & 10\% & 10\% & 10\% \\
  \multicolumn{2}{l}{Seqlen} & 2048 & 4096 & 4096 \\
  \multirow{4}{*}{\#Param.} & $d_{kv}=16$ & 1598.30M (24.62\%) & - & - \\
  & $d_{kv}=32$ & 1967.40M (28.67\%) & 1225.02M (14.91\%) & 809.21M (9.83\%) \\
  & $d_{kv}=64$ & 2705.60M (35.60\%) & 1321.48M (15.90\%) & 841.32M (10.18\%) \\ 
  & $d_{kv}=128$ & - & 1514.42M (17.80\%) & 905.55M (10.87\%) \\
  \arrayrulecolor{black}
  \bottomrule
\end{tabular}
\caption{Training configurations and model parameter summary across models for stage 2.}
\label{tab:s2_hp}
\end{table*}

\subsubsection{Evaluation Setups}

\paragraph{Benchmarks}
To comprehensively assess the capabilities of MHA2ML VLMs or baselines, we adopt eight widely-used benchmarks that span diagram reasoning, general visual QA, object hallucination, scene understanding, real-world images, multi-modal comprehension, chart reasoning, and document understanding—namely including AI2D~\cite{DBLP:conf/eccv/KembhaviSKSHF16}, GQA~\cite{DBLP:conf/cvpr/HudsonM19}, POPE~\cite{DBLP:conf/emnlp/LiDZWZW23}, SEED-Bench-IMG~\cite{li2023seedbenchbenchmarkingmultimodalllms}, RealWorldQA~\cite{DBLP:conf/iclr/0004ZTFZWLWW00025}, MMBench~\cite{DBLP:conf/eccv/LiuDZLZZYWHLCL24}, ChartQA~\cite{DBLP:conf/acl/MasryLTJH22}, DocVQA~\cite{DBLP:conf/wacv/MathewKJ21}.

\paragraph{Model-specific availability} For LLaVA-1.5, we follow its original training and evaluation protocol and therefore exclude ChartQA and DocVQA. The model’s pre-training and fine-tuning corpus lack chart and document data, making direct evaluation on these two tasks ill-posed and unfair. The other two models (LLaVA-NeXT and Qwen2.5-VL) are evaluated on the full set of eight benchmarks.

In summary, the types of VLMs we evaluated are as follows:

\begin{itemize}[leftmargin=*,itemsep=0pt, topsep=0pt, parsep=0pt]
    \item The officially released checkpoints from LLaVA-1.5, LLaVA-NeXT and Qwen2.5-VL.
    \item Original VLMs fine-tuned on the datasets we used.
    \item Our MHA2MLA-VLM models with different $d_{kv}$ settings.
\end{itemize} 

\paragraph{Main Experiments and Ablation Studies}

For the main experiments in ~\Cref{tab:main_result} and ~\Cref{tab:compare_cache}, we used our proposed PEFT training strategy for the baselines and MHA2MLA-VLM.
For the ablation experimentsin ~\Cref{tab:ablation_study}, to isolate the impact of each component for a fair comparison, we fine-tuned all parameters of the models.

\end{document}